\documentclass[10pt,journal,compsoc]{IEEEtran}

\ifCLASSOPTIONcompsoc
  \usepackage[nocompress]{cite}
\else
  \usepackage{cite}
\fi

\ifCLASSINFOpdf
   \usepackage[pdftex]{graphicx}
\else
\fi

\usepackage[utf8]{inputenc} 
\usepackage[T1]{fontenc}    
\usepackage{xr-hyper}
\usepackage{hyperref}       
\usepackage{url}            
\usepackage{booktabs}       
\usepackage{amsfonts}       
\usepackage{nicefrac}       
\usepackage{microtype}      
\usepackage{subfigure}
\usepackage{wrapfig}

\usepackage{amsmath,amsfonts,amssymb,amsthm,dsfont,stmaryrd}

\usepackage{amssymb}

\usepackage{color}
\usepackage{multirow}

\def\vol{\mathrm{vol}}
\def\cost{\mathrm{cost}}

\def\R{\mathbb{R}}

\def\our{OneFlow}
\def\ourdet{OneFlow-Gen}
\def\flow{LL-Flow}
\def\flowdet{LL-Flow-Gen}

\newtheorem{observation}{Observation}

\theoremstyle{remark}
\newtheorem{problem}{Problem}
\newtheorem{remark}{Remark}

\begin{document}

\title{OneFlow: One-class flow for anomaly detection based on a minimal volume region}

\author{Łukasz Maziarka, 
        Marek Śmieja, 
        Marcin Sendera, 
        Łukasz Struski,
        Jacek Tabor,
        Przemysław Spurek
\IEEEcompsocitemizethanks{\IEEEcompsocthanksitem Ł. Maziarka, M. Śmieja, M. Sendera, Ł. Struski, J.Tabor, P. Spurek are  with  the  Faculty of Mathematics and Computer Science, Jagiellonian University, Łojasiewicza 6,30-348 Kraków, Poland.\protect\\
E-mail: \href{mailto:lukasz.maziarka@ii.uj.edu.pl}{lukasz.maziarka@ii.uj.edu.pl}, \href{mailto:marek.smieja@uj.edu.pl}{marek.smieja@uj.edu.pl}}
\thanks{Manuscript received April 19, 2005; revised August 26, 2015.}}

\markboth{Journal of \LaTeX\ Class Files,~Vol.~14, No.~8, August~2015}%
{Shell \MakeLowercase{\textit{et al.}}: Bare Demo of IEEEtran.cls for Computer Society Journals}

\IEEEtitleabstractindextext{%
\begin{abstract}
We propose OneFlow -- a flow-based one-class classifier for anomaly (outlier) detection that finds a minimal volume bounding region. Contrary to density-based methods, OneFlow is constructed in such a way that its result typically does not depend on the structure of outliers. This is caused by the fact that during training the gradient of the cost function is propagated only over the points located near to the decision boundary (behavior similar to the support vectors in SVM).
The combination of flow models and a Bernstein quantile estimator allows OneFlow to find a parametric form of bounding region, which can be useful in various applications including describing shapes from 3D point clouds. Experiments show that the proposed model outperforms related methods on real-world anomaly detection problems.

\end{abstract}

}

\maketitle

\IEEEdisplaynontitleabstractindextext

\IEEEpeerreviewmaketitle

\IEEEraisesectionheading{\section{Introduction}\label{sec:introduction}}

\IEEEPARstart{A}{nomaly} (novelty/outlier) detection refers to the identification of abnormal or novel patterns embedded in a large amount of (nominal) data \cite{miljkovic2010review}. The goal of anomaly detection is to identify unusual system behaviors, which are not consistent with its typical state. Anomaly detection algorithms find application in fraud detection \cite{phua2010comprehensive}, discovering failures in industrial domain \cite{lavin2015evaluating}, detection of adversarial examples \cite{roth2019odds}, etc. \cite{garcia2009anomaly, shone2018deep, goh2017anomaly}.

In contrast to typical binary classification problems, where every class follows some probability distribution, an anomaly is a pattern that does not conform to the expected behavior. 
In other words, a completely novel type of outliers, which is not similar to any known anomalies, can occur at a test time. Moreover, in most cases, we do not have access to any anomalies at training time. 
In consequence, novelty detection is usually solved using unsupervised approaches, such as one-class classifiers, which focus on describing the behavior of nominal data (inliers) \cite{scholkopf2001estimating, abati2019latent, li2018anomaly, wang2019effective}. Any observation, which deviates from this behavior, is labeled as an outlier.

Following the above motivation, we propose \our{} -- a deep one-class classifier based on flow models. In contrast to typical (generative) flow-based density models, which focus on density estimation,
\our{} does not depend strongly on the structure of outliers since it finds a bounding region with a minimal volume for a fixed $(1-\alpha)$ portion of data, e.g. $\alpha=5\%$, see Figure \ref{fig:sat} for a comparison. This is realized by finding a hypersphere with a minimal radius in the output space of a neural network, which contains $(1-\alpha)$ percentage of data,  see Figure \ref{fig:intro}. 
While minimum volume sets were considered previously in \cite{scott2006learning, zhao2009anomaly} in the context of anomaly detection, these works were mainly devoted to theoretical aspects and the algorithms proposed there do not scale well to large datasets. On the other hand, in the case of kernel methods, such as Support Vector Data Description (SVDD) \cite{tax2004support}, the minimum volume problem has been reformulated to obtain a convex objective, which solves a slightly different problem.
In particular, instead of enclosing $(1-\alpha)$ fraction of data within a hypersphere, SVDD adds a penalty for data points outside the hypersphere. Making use of neural networks, we do not need to stick to convex optimization and, therefore, \our{} solves the original problem directly. In \our{}, the gradient of the cost function is propagated only through the points located near the decision boundary (a behaviour which is similar to that of support vectors in SVM), see Remark \ref{re:31}. To the authors' best knowledge, this is the first model, which applies this paradigm in deep learning without any relaxations described above.

\begin{figure}[t]
  \begin{center}
    \includegraphics[width=0.4\textwidth]{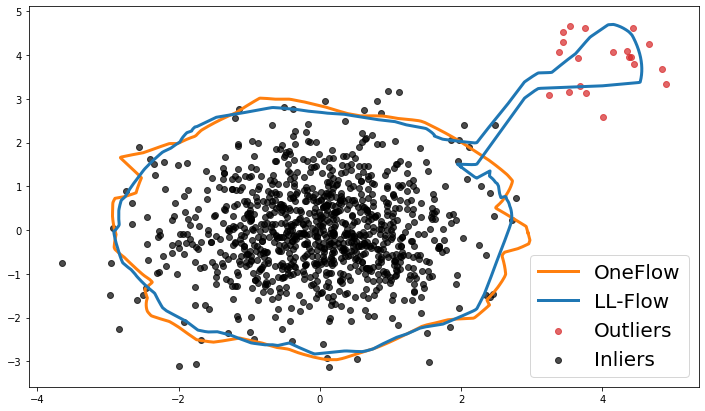}
  \end{center}
  \caption{Bounding regions constructed by \our{} and typical log-likelihood flow-based density model (\flow{}). \flow{} puts a similar weight to both blobs and marks a few examples from the smaller one as nominal data. \our{} finds a bounding region with a minimal volume for a fixed percentage of data. To minimize the volume it focuses on a bigger blob and considers the smaller one as anomalies.}\label{fig:sat}
\end{figure}

\begin{figure*}[t] 
    \centering
        \includegraphics[width=0.8\textwidth]{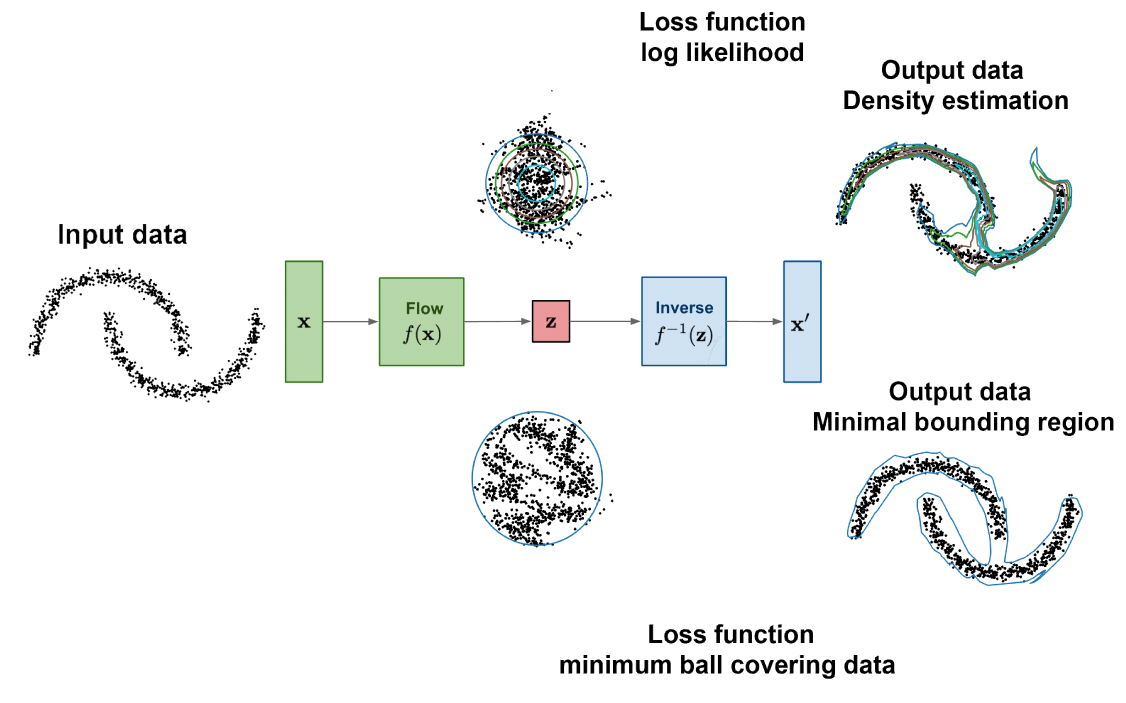} 
    \caption{\our{} finds a bounding region with a minimal volume for a fixed percentage of data using a hypersphere in the latent space of flow model.}
    \label{fig:intro}
\end{figure*}

\begin{figure}[t] 
    \centering
        \includegraphics[width=0.22\textwidth]{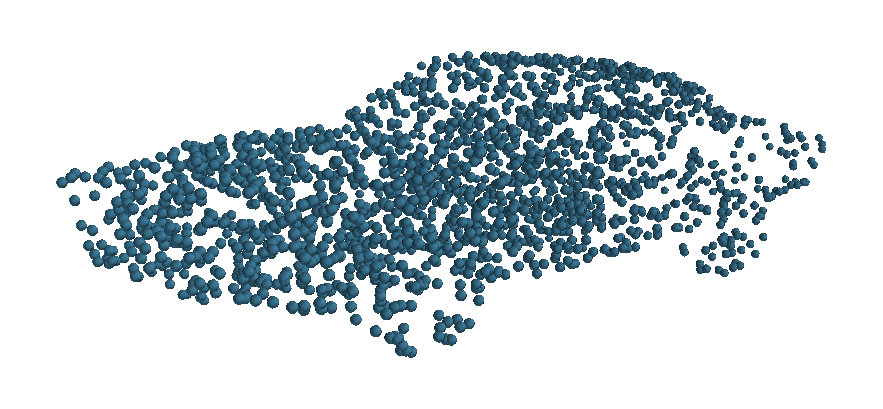}
        \includegraphics[width=0.22\textwidth]{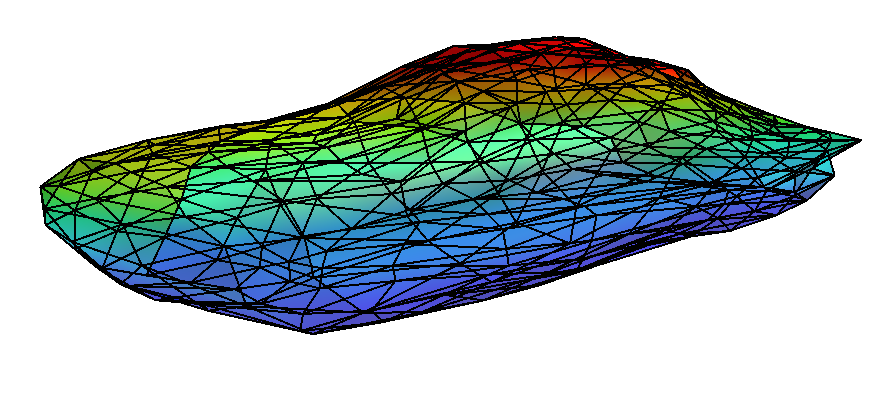}\\
        \includegraphics[width=0.22\textwidth]{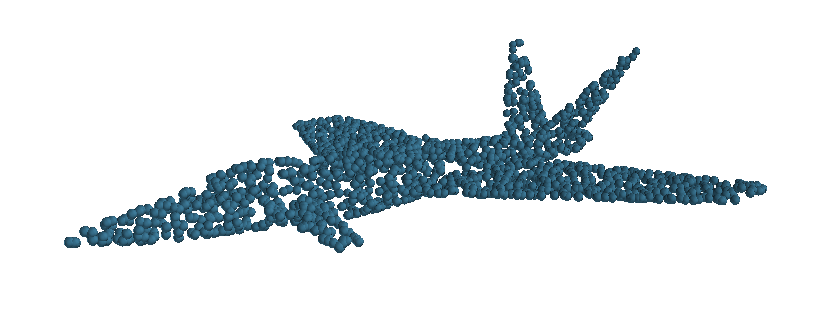}
        \includegraphics[width=0.22\textwidth]{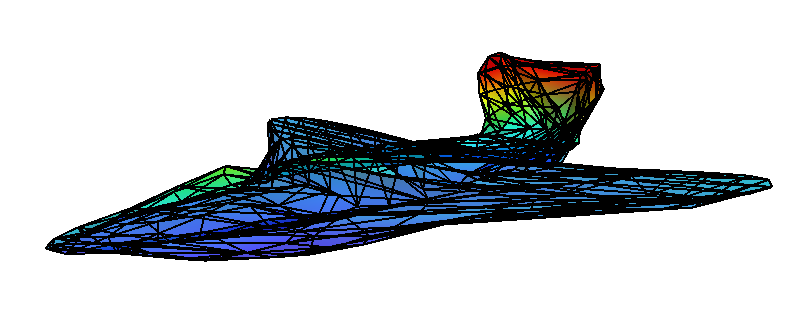}
    \caption{
    Mesh representations generated by \our{} (right) for the shapes represented as 3D point clouds (left). Our method automatically removes outliers, which may be generated from other shapes in the background, and gives an explicit parametric form of the boundary of objects using the inverse mapping of the flow model.}
    \label{fig:plots_3d_car1_1}
\end{figure}

\our{} uses two important ingredients. The first one is the application of flow-based models~\cite{dinh2014nice, kingma2018glow}, which give an explicit formula for the inverse mapping and allow us to calculate a Jacobian of a neural network at every point. In consequence, minimizing the volume of the hypersphere in the feature space leads to the minimization of the volume of the corresponding bounding region in the input space. Moreover, making use of the inverse mapping, we automatically get a parametric form for the corresponding bounding region in the input space, which is useful, for example, in describing shapes from 3D point clouds~\cite{yang2019pointflow,spurek2020hypernetwork}, see Figure \ref{fig:plots_3d_car1_1} for details. The second ingredient is the Bernstein polynomial estimator \cite{cheng1995bernstein} of the upper $(1-\alpha)$-quantile, which is used in \our{} loss to estimate a hypersphere for $(1-\alpha)$ fraction of data. 

We performed extensive experiments on image datasets (MNIST and Fashion-MNIST), examples retrieved from UCI~\cite{asuncion2007uci}, openML~\cite{vanschoren2014openml} repositories as well as KDD CUP 1999. Experimental results show that \our{} gives comparative or even better performance than state-of-the-art models for anomaly detection. In particular, \our{} outperforms typical flow-based density models, which use log-likelihood objective as well as deep SVDD method. 

Our contribution is summarized as follows:
\begin{enumerate}
    \item We introduce a one-class model for anomaly detection based on invertible flows, which focuses on finding a bounding region for nominal data.
    \item We show that the application of Bernstein quantile estimator in our neural network framework allows us to estimate the volume of the bounding region for a fixed $(1-\alpha)$ percentage of data in a closed-form.
    \item We experimentally analyze the behavior of the proposed approach and compare it with state-of-the-art methods.
\end{enumerate}

The code is available at \url{github.com/gmum/OneFlow}.


\section{Related works}

One of the most successful approaches to anomaly detection is based on one-class learning. One-class SVM (OCSVM) \cite{scholkopf2001estimating} and SVDD \cite{tax2004support} are two well known kernel methods for one-class classification. While OCSVM directly uses SVM to separate the data from the origin (considered as the only negative sample), SVDD aims to enclose most of the data points inside a hypersphere with minimal volume. To provide a unique local minimum, SVDD relaxes this problem to the convex objective by penalizing data points outside the hypersphere. In a similar spirit, \cite{chen2013new} apply Ranking SVM based on rankings created from pairwise comparison of nominal data. In contrast to SVDD, which reformulates minimum volume problem, \cite{scott2006learning, zhao2009anomaly} derived algorithms, which, under some assumptions on data density, are provably optimal. However, despite obtaining important theoretical results, these methods do not scale well to large datasets. In contrast to these works, we do not use any simplifications in the cost function, but solve the minimum volume set problem efficiently.

Recent research on anomaly detection is dominated by methods based on deep learning. Attempts for adapting SVDD to the case of neural networks are presented in \cite{ruff2018deep, kim2015deep, ruff2019deep, chong2020simple}. However, the direct minimization of SVDD loss may lead to hypersphere collapse to a single point. To avoid this negative behavior, DSVDD (deep SVDD) recommends that the center must be something other than the all-zero-weights solution, and the network should use only unbounded activations and omit bias terms \cite{ruff2018deep}. While the first two conditions can be accepted, omitting bias terms in a network may lead to a sub-optimal feature representation due to the role of bias in shifting activation values. To eliminate these restrictions a recent work \cite{chong2020simple} proposes two regularizers, which prevent from hypersphere collapse, and use an adaptive weighting scheme to control the amount of penalization between the SVDD loss and the respective regularizer. Making use of flow models, we are able to directly calculate the volume of the bounding region in the original space (not only the volume of the transformation in the latent space). Due to the direct correspondence between original density and prior density given by a flow model, we avoid degenerate solutions, which may appear in DSVDD models.

The vast majority of deep learning methods use neural network representation learning capability to generate a latent representation to preserve the details of the given class based on auto-encoder reconstruction error \cite{abati2019latent, dasgupta2018neural, li2018anomaly}. This line of research includes strictly unsupervised techniques \cite{wang2019effective} as well as supervised and semi-supervised methods \cite{shu2018unseen}. Nevertheless, there is no theoretical justification that reconstruction error captures enough information to separate nominal data from outliers. 
Many research proposes GAN-based approaches as the criterion considered introduced anomaly score \cite{schlegl2017unsupervised, schlegl2019f}, sample's representation in the discriminator's latent space \cite{deecke2018image, perera2019ocgan} or even separability of the enhanced inliers and distorted outliers
\cite{sabokrou2018adversarially}. Another direction of related work is the discovery of out-of-distribution instances (which are basically anomalies). Various forms of thresholding are used on the classification output to detect anomalies \cite{hendrycks2016baseline, liang2017enhancing, devries2018learning}. \cite{wang2019multivariate} defined a general approach for anomaly detection, which is based on thresholding multivariate quantile function. Analogically to our approach, they use flow-based models, but in the context of density estimation. Another use of density-based flow models is presented in \cite{schmidt2019normalizing}. Moreover, there are methods, which uses flow-based models with their ability to density estimation as key part of anomaly detection models - \cite{nachman2020anomaly, wellhausen2020safe}. A comprehensive review on deep anomaly detection can be found in \cite{ruff2021unifying, pang2020deep}.

\section{The proposed model}

In this section, we first give a precise formulation of outlier detection task using a bounding region with minimum volume. Next, we present our solution based on flow models. Finally, we discuss optimization issues, which are solved using the Bernstein quantile estimator. The notation used in the following section is summarized in Appendix \ref{app:notation}.

\subsection{Problem formulation.} Our goal is to find a bounding region with a minimal volume, which contains a fixed amount of data, e.g. $95\%$. This refers to one of the typical ideas used in one-class classification \cite{scott2006learning, zhao2009anomaly, tax2004support}, where we describe the behavior of nominal data. Ignoring a small number of data allows us to deal with anomalies in training as well as to focus only on the most important features. In contrast to density-based approaches, which estimate a density of the whole data, we solve the easier task by separating nominal data from abnormal examples.

Let $g$ be a density in $\R^D$, and let $X \subset \R^D$ be a sample generated by $g$. We assume that we are given a $p$-value $\alpha \in (0,1)$, which determines the percentage of possible outliers, which can also be seen as a false alarm rate\footnote{
If not stated otherwise, we use $\alpha=5\%$, which is motivated by a typical approach used in hypothesis testing.}. We say that $U \subset \R^D$ is a $(1-\alpha)$-bounding region of $g$ if a data point generated from a density $g$ belongs to $U$ with a probability $1-\alpha$, i.e. $\int_U g(x) \, dx=1-\alpha$. Intuitively, the $(1-\alpha)$-bounding region covers approximately $(1-\alpha)$ percentage of data, which allows us to deal with outliers or noise in training data.

Our problem is formally formulated below:
\begin{problem}
Find a $(1-\alpha)$-bounding region $U \subset \R^D$ with a minimal volume for a density $g$ generating data, i.e.
    $$
    \arg \min_U \big\{\vol(U) \, | \, U:
    \int_U g(x) \, dx=1-\alpha\big\}.
    $$
\end{problem}

To allow for sufficient flexibility in defining the form of $U$, we use deep neural networks. Given a neural network $f: \R^D \to \R^D$, we aim at finding $r > 0$ such that $U := f^{-1}(B(0,r))$, where $B(0,r)$ denotes a ball centered at $0$ with radius $r$ in the latent (feature) space of the neural network. In other words, the $(1-\alpha)$-bounding region $U$ is the inverse of a ball with radius $r$ in the feature space. While the computation of $f^{-1}$ can be difficult for arbitrary neural networks, we restrict our attention to flow-based models, which give an explicit form of inverse mapping $f^{-1}$. 

First, we demonstrate that the volume $\vol(f^{-1}(B(0,r)))$ can be calculated efficiently for flow-based models. Next, we show that the application of Bernstein estimator allows us to find a hypersphere for $(1-\alpha)$ percentage of data, which is the solution of our optimization problem.

\subsection{Volume calculation using flow-based models.}

Let us recall that a neural network $f:\R^D \to \R^D$ is a flow-based model if the inverse mapping $f^{-1}$ is given explicitly and the Jacobian determinant $\det d f(x)$ can be easily calculated. Flow-based models have been usually used in the case of generative models because a direct form of $f^{-1}$ allows one to generate samples from the prior distribution, while the condition for Jacobian makes the optimization of log-likelihood function possible. Their direct application in the context of anomaly detection can be compared to the use of GMMs. Given a distribution of data, we discard $\alpha$ percentage of data or a region of data space with a probability $\alpha$. Since we want to realize a different objective, we need to redefine the loss for flow-based models. 

As mentioned, flow-based models are designed to calculate the Jacobian of $f$ effectively, which allows us to optimize the log-likelihood function in the case of neural networks directly. From this perspective, flow-based models can be divided into two natural classes. The first class, referred to as \emph{const-det flows}, contains models where $\det d f(x)$ is constant (does not depend on $x$), e.g. NICE \cite{dinh2014nice}. The models from the second class, called here \emph{general flows}, can change the derivative at different points, e.g. Real NVP \cite{dinh2016density}.

We show that for const-det flows we can obtain the exact formula for $\vol(f^{-1}(B(0,r)))$, while for general flows its approximation can be derived. For this purpose, we introduce a notation:
$$
w(x) = \det d (f^{-1})(x) = \frac{1}{\det d f (f^{-1}(x))}.
$$
Note that for const-det flows, $w(x) = w$ is a constant function.

\begin{observation}
Let $f$ be a const-det flow model, i.e. $w = w(x)$ is constant. The volume of $U = f^{-1}(B(0,r))$ is given by:
\begin{equation}\label{eq:const}
\vol(U) = \vol(B(0,1)) \cdot w \cdot r^D.
\end{equation}
\end{observation}
\begin{proof}
Clearly, by a change of variables,
$$
\vol(U)=\vol (f^{-1}(B(0,r)))=\int_{B(0,r)} \det d(f^{-1})(x) \, dx.
$$
Since $w$ is a constant function, we get
$\vol(U)=\int_{B(0,r)} w(x) \, dx= \vol(B(0,1)) \cdot  w \cdot r^D$.
\end{proof}

If $w(x)$ depends on $x$, then the situation is more complex, but we can still obtain an approximation of the volume for general flows as
\begin{align} \label{eq:var}
\vol(U)=\int_{B(0,r)} w(x) dx 
\approx \vol(B(0,r)) \cdot \frac{1}{m}\sum_{k=1}^m w(r \cdot e_i) \nonumber\\
=\frac{\vol(B(0,1))}{m} \cdot r^D \cdot \sum_{k=1}^m w(r \cdot e_i),
\end{align}
where $e_i$ are points randomly chosen with respect to uniform distribution on $B(0,1)$. This is a type of the Monte Carlo sampling and it is generally difficult to control the accuracy of this estimation

\subsection{Optimization algorithm.}
We presented how to find formulas for computing the volume of the bounding region using flow-based models. Now, we apply this fact to construct the optimization procedure for computing $(1-\alpha)$-bounding region. 

Let $(x_1,\ldots,x_n) \subset X$, be a mini-batch, $\theta$ be the weights of the flow model $f_\theta$ and $\alpha \in (0,1)$ be a given $p$-value. To apply the formula \eqref{eq:const} for const-det flow, we need to find the radius of the ball, which contains $(1-\alpha)$ percentage of data. We estimate this radius by first computing 
$$
r_i(\theta)=\|f_\theta(x_i)\|, 
$$
and next applying the estimator of upper $(1-\alpha)$-quantile. 

As a quantile estimator, we use Bernstein polynomial estimator \cite{cheng1995bernstein, leblanc2012estimating, zielinski2004optimal} -- see Remark \ref{thm:rem} for the justification of this selection. Let us recall that given a sample $s_1,\ldots,s_n$ drawn from the same distribution, the Bernstein estimation of $(1-\alpha)$-quantile $s_\alpha$, where $\alpha \in (0,1)$, is constructed in the following way. First, we reorder $s_i$ so that $s_{(1)} \leq \ldots \leq s_{(n)}$. Then the Bernstein estimator of $(1-\alpha)$-quantile is defined by:
$$
s_\alpha = \sum_{k=1}^n \binom{n-1}{k-1} \alpha^{k-1}(1-\alpha)^{n-k} \cdot s_{(k)}
$$
Bernstein polynomials are known to yield very smooth estimates, even from the small sample size, that typically have acceptable behavior at the boundaries.

Applying the above construction to our case, we do as follows:
\begin{itemize}
    \item the sequence $(r_{(i)}(\theta))$ is obtained by sorting $r_i(\theta)$ in a descending order,
    \item the Bernstein polynomial estimator of upper $(1-\alpha  )$-quantile is given by 
    \begin{equation} \label{eq:R}
    R_\alpha(\theta)= \sum_{k=1}^n \binom{n-1}{k-1}\alpha^{k-1} (1-\alpha)^{n-k} \cdot r_{(k)}(\theta).
    \end{equation}
\end{itemize}

In consequence, the volume of the bounding region for const-det flows is given by 
$$
\vol=w \cdot \vol(B(0,1)) \cdot R_\alpha(\theta)^D.
$$
To avoid potential numerical problems in the cost function, one can minimize logarithm of the volume (instead of the volume itself):
$$
\cost(\theta)=\log(\vol(B(0,1)))+\log(w) +D \log  R_\alpha(\theta).
$$

For general flows, we use the formula \eqref{eq:var} in the above calculations. Thus the estimation for the volume of the $(1-\alpha)$-bounding region is given by 
$$
\vol \approx \tfrac{\vol(B(0,1))}{m} \cdot R^D_\alpha(\theta) \cdot \sum_{k=1}^m w(R_\alpha(\theta) e_k),
$$
where  $(e_k)$ is a sequence of $m$ randomly chosen points from the uniform distribution on the unit ball $B(0,1)$. 
The final cost function in the logarithmic form equals:
\begin{multline*}
\cost(\theta)=\\
\log(\tfrac{\vol(B(0,1))}{m}) + D \log R_\alpha(\theta) + \log \big( \sum_{k=1}^m w(R_\alpha(\theta) e_k) \big).
\end{multline*}

\begin{remark} \label{re:31}
We explain now that, in contrast to flow-based density models, the gradient of our loss function is propagated only over a small number of points, which are located close to the decision boundary. For $n\alpha>9$ (e.g. for $\alpha=0.05$ and $n=256$), the Bernoulli distribution can be approximated by the Gaussian distribution $N(m,\sigma^2)$ with $m=n\alpha$ and $\sigma^2=n\alpha(1-\alpha)$. Thus by the $3\sigma$ law, we obtain that numerically essential weights are only for $k$ examples, where $k \in [n\alpha-3\sqrt{n\alpha(1-\alpha)},
 n\alpha+3\sqrt{n\alpha(1-\alpha)}]$. Consequently, we obtain that only the following percentage of samples from the batch obtain nonzero gradient:
$$ 
\frac{6 \sqrt{n \alpha (1-\alpha)}}{n}=
\frac{3\sqrt{19}}{160} \approx 0.08,
$$
where $6 \sqrt{n \alpha (1-\alpha)}$ is the length of the above interval.
\end{remark}

\begin{remark} \label{thm:rem}
The simplest estimator of the $(1-\alpha)$-quantile for the sorted sample $s_{(1)}\leq \ldots \leq s_{(n)}$ can defined by:
$$
s^{sim}_{\alpha} = \left\{
\begin{array}{ll}
     s_{(\alpha n)}, &  \text{ if } \alpha n \text{ is an integer,} \\
     s_{([\alpha n]+1)}, &  \text{ otherwise, } 
\end{array}
\right.
$$
where $[a]$ is the greatest integer which is not greater than $a \in \R$. If such an estimator was used in our loss function, then the gradient would be based only on a single point that establishes the gradient. This may result in a very slow and inefficient training. In the case of Bernstein polynomial estimator, the gradient of the loss function in \our{} is propagated over a number of points, which are located close to the decision boundary (see Remark~\ref{re:31}) -- the behavior that is similar  to SVM models. Another thing is that the Bernstein estimator is smooth \cite[Section 2.3]{zielinski2004optimal}, while the simple estimator $s^{sim}_{\alpha}$ is not even continuous.

It is also instructive to see the difference between \our{} and typical flow-based method applying log-likelihood loss. Since the log-likelihood method focuses on a density estimation, the gradient is based on all points, which could lead to an overfitting and not optimal setting of the border. Density estimation can be seen as an additional task in anomaly detection, while the essential problem is to estimate the quantile, which is directly solved by \our{}.
\end{remark}


\section{Experiments}

In this section, we experimentally examine \our{} and compare it with several state-of-the-art approaches. \our{} is implemented using the architecture of NICE flow model
and $\alpha = 5\%$ (see Appendix~\ref{app:exp} for the experimental setting). An analysis of parameter $\alpha$ is presented at the end of this section. If not stated otherwise, we consider a variant of const-det flow (Jacobian determinant is constant).  

\begin{table*}[t]
\caption{Performance on two anomaly detection datasets (measured by Precision, Recall and F1 score). The results marked with * are taken from \cite{wang2019multivariate}.} \label{tab:res1}
\centering
\begin{tabular}{l ccc ccc }
\toprule
{} & \multicolumn{3}{c}{Thyroid} & \multicolumn{2}{c}{KDDCUP} \\ 
\cline{2-7}
{} & Precision & Recall & F1 & Precision & Recall & F1 \\
\midrule
OC-SVM *        & .3639 & .4239 & .3887 & .7457 & .8523 & .7954 \\
DSEBM *         & .0404 & .0403 & .0403 & .7369 & .7477 & .7423 \\
DAGMM *         & .4766 & .4834 & .4782 & .9297 & .9442 & .9369 \\
DSVDD          & .6989 & .6989 & .6989 & .6898 & .7055 & .6975 \\ 
NLL *           & .7312 & .7312 & .7312 & .9622 & .9622 & .9622 \\
TQM$_{1}$ *     & .5269 & .5269 & .5269 & .9621 & .9621 & .9621 \\ 
TQM$_2$ *       & .5806 & .5806 & .5806 & .9622 & .9622 & .9622 \\
TQM$_{\infty}$ * & .7527 & .7527 & .7527 & .9622 & .9622 & .9622 \\
LL-Flow        & .6989 & .6989 & .6989 & .6782 & .6524 & .6650 \\
LL-Flow-Gen    & .6808 & .6955 & .6881 & .9268 & .9268 & .9268 \\
\our{}         & \textbf{.7634} & \textbf{.7634} & \textbf{.7634} & .9702 & .9702 & .9702 \\
\ourdet{}      & .7097 & .7097 & .7097 & \textbf{.9712} & \textbf{.9712} & \textbf{.9712} \\
\bottomrule
\end{tabular}
\end{table*}

\subsection{Benchmark data for anomaly detection.} 
First, we provide a quantitative assessment and take into account the Thyroid\footnote{\url{http://odds.cs.stonybrook.edu/thyroid-disease-dataset/}} and KDDCUP\footnote{\url{http://kdd.ics.uci.edu/databases/kddcup99/kddcup.testdata.unlabeled_10_percent.gz}} datasets, which are real-world benchmark datasets for anomaly detection. We use the standard training and test splits and follow exactly the same evaluation protocol as in \cite{wang2019multivariate}. The performance was measured using F1 score, because this metric was reported for all methods considered.

\begin{figure*}[t] 
    \centering
    \subfigure{%
  	    \centering
        \includegraphics[width=0.48\textwidth]{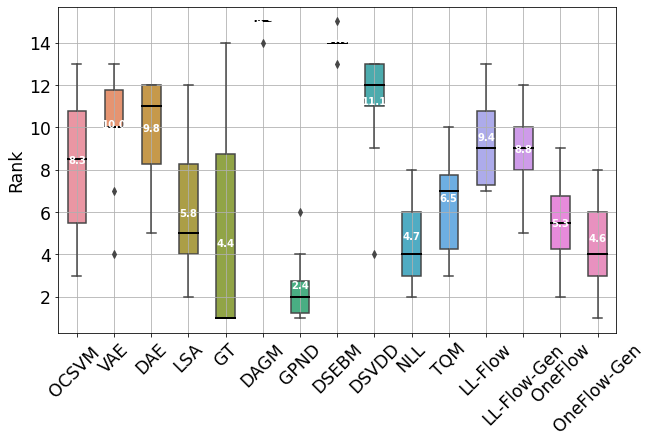}
        \label{fig:ranks_mnist}
      } 
    \subfigure{%
      	\centering
        \includegraphics[width=0.48\textwidth]{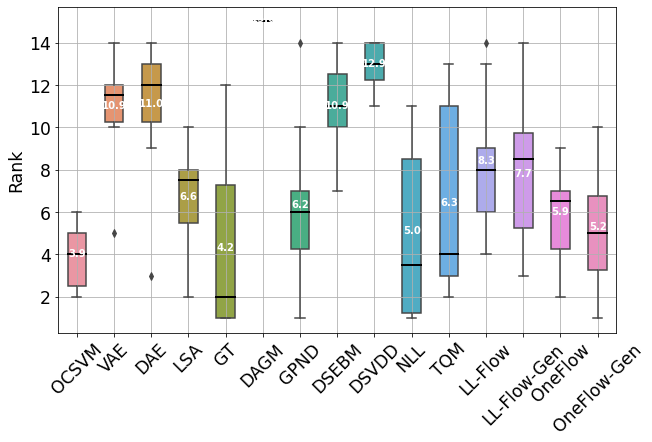}
        \label{fig:ranks_fmnist}
      } 
    \caption{Box plots for rankings calculated on MNIST (left) and Fashion-MNIST (right) using AUC score. The median ranking is marked by a line, while the average ranking is marked with a number. The results of cometitive methods (except DSVDD, LL-Flow and LL-Flow-Gen) are taken from \cite{wang2019multivariate}.}
    \label{fig:ranks}
\end{figure*}

We use two variants of \our{}. The first one (\our{}) uses constant Jacobian while the second (\ourdet{}) allows for changing the Jacobian at every point. Our models are compared with the following algorithms (see Appendix~\ref{app:desc} for more detailed description): One-class SVM (OC-SVM) \cite{scholkopf2001estimating}, Deep structured energy-based models (DSEBM) \cite{zhai2016deep}, Deep autoencoding Gaussian mixture model (DAGMM) \cite{zong2018deep}, variants of MQT -- multivariate quantile map (NLL, TQM$_1$, TQM$_2$, TQM$_\infty$) \cite{wang2019multivariate}, Deep Support Vector Data Description (DSVDD) \cite{ruff2018deep}, two variants of log-likelihood flow model (\flow{} and \flowdet{}). The bounding region of \flow{} is constructed by taking the smallest hypersphere in the latent space, which covers $(1-\alpha)$ percentage of data (the hypersphere is determined by the prior Gaussian distribution). Due to the change of variable rule used in flow models, this strategy corresponds to thresholding the density in the original space so that to cover $(1-\alpha)$ percentage of data.

The results presented in Table \ref{tab:res1} show that both variants of our model perform almost equally on KDDCUP and they are better than all competitive methods. In the case of Thyroid, \our{} presents the best performance, while \ourdet{} gives third best score. The most similar method, DSVDD, was not able to obtain similar performance on these datasets, which shows that a direct minimization of the bounding region implemented by our method is more beneficial. While the results of Triangular Quantile Maps (TQM and NLL) depends heavily on the assumed norm and cost function, there is no objective criteria for selecting these parameters. Other methods produce worse results.

\subsection{Image datasets}

To provide further experimental verification, we use two image datasets: MNIST and Fashion-MNIST. In contrast to the previous comparison, these two datasets are usually used for multiclass classification and thus need to be adapted to the problem of anomaly detection. For this purpose, each of the ten classes is deemed as the nominal class while the rest of the nine classes are deemed as the anomaly class, which results in 10 scenarios for each dataset. To be consistent with \cite{wang2019multivariate}, we report AUC (area under ROC curve).

We additionally compare with the following models: Geometric transformation (GT) \cite{golan2018deep}, Variational autoencoder (VAE) \cite{kingma2013auto}, Denoising autoencoder (DAE) \cite{vincent2008extracting}, Generative probabilistic novelty detection (GPND) \cite{pidhorskyi2018generative}, Latent space autoregression (LSA) \cite{abati2019latent}. In contrast to previous experiment, we only use TQM$_2$ and NLL as the only implementations of MTQ, because they output the highest value of AUC \cite{wang2019multivariate}.

To present the results, we compute the ranking on each of 10 scenarios and summarize it using box plot, see Figure \ref{fig:ranks} (detailed results are included in Appendix~\ref{app:detailed_results} in Tables~\ref{tab:mnist} and~\ref{tab:fashion}). It is evident that \our{} and \ourdet{} outperform related \flow{} and \flowdet{}, which confirms that the proposed loss function suits better for one-class classification problems than typical log-likelihood function. Our methods give also better scores than DSVDD, which implements a similar loss function. The overall ranking of \our{} and \ourdet{} is comparative to the best performing methods on both datasets. It is difficult to clearly determine which method performs best, because of the high variation in the results. While GPND seems to outperform other methods on MNIST, its result on Fashion-MNIST is similar to \our{}. We emphasize, however, that both MNIST and Fashion-MNIST do not represent typical anomaly detection datasets.

\begin{figure}[t]
  \begin{center}
    \includegraphics[width=0.48\textwidth]{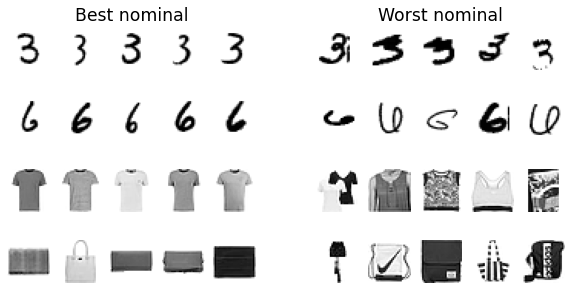}
  \end{center}
  \caption{Best nominal (left) and worst nominal (right) examples determined by OneFlow for MNIST (top) and Fashion-MNIST (bottom). }\label{fig:samples}
\end{figure}

\begin{figure*}[h!] 
    \centering
    \includegraphics[width=0.45\textwidth]{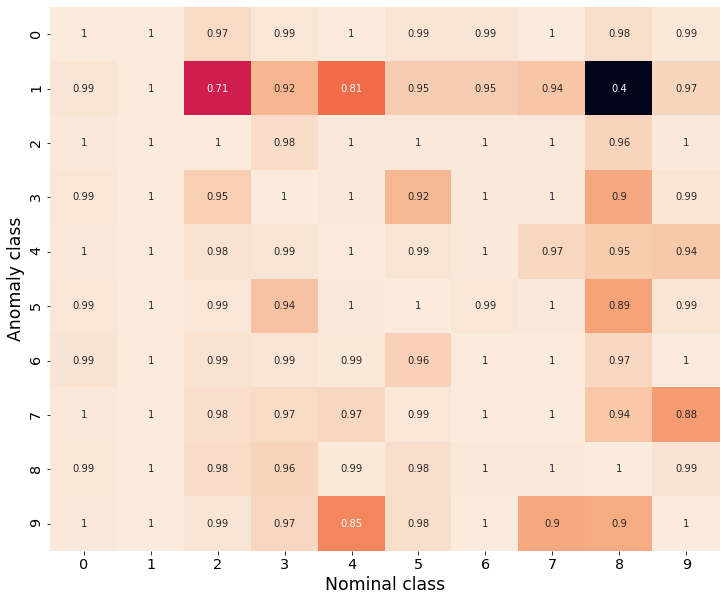}
    \includegraphics[width=0.45\textwidth]{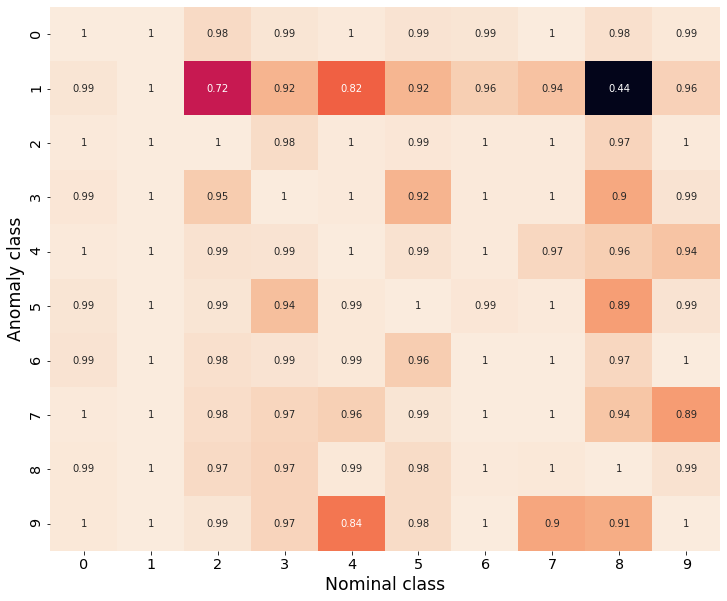}
    \caption{AUC obtained for one nominal class (in columns) and one anomaly class (in rows) by \our{} (left) and \ourdet{} (right) on MNIST dataset.}
    \label{fig:datailed_scores_mnist}
\end{figure*}

\begin{figure*}[h!] 
    \centering
    \includegraphics[width=0.45\textwidth]{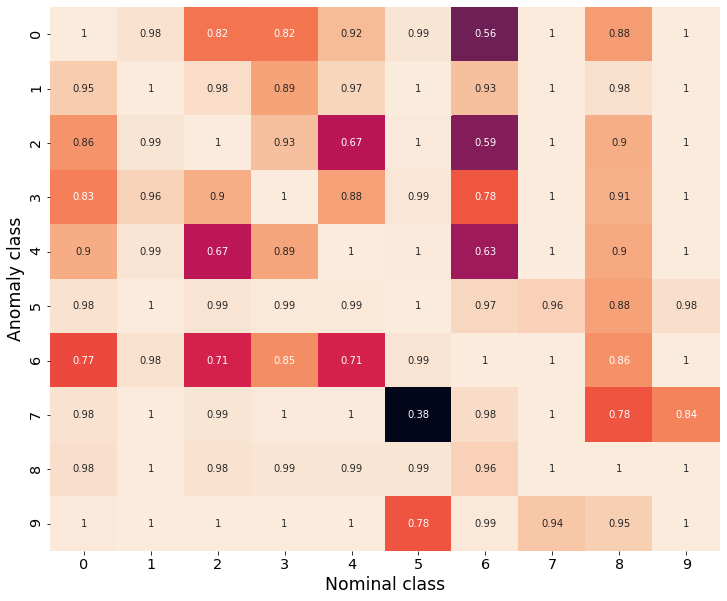}
    \includegraphics[width=0.45\textwidth]{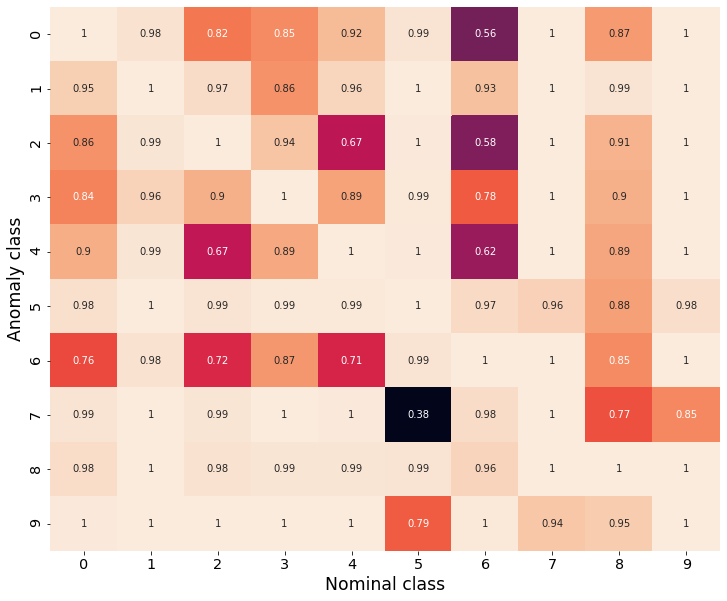}
    \caption{AUC obtained for one nominal class (in columns) and one anomaly class (in rows) by \our{} (left) and \ourdet{} (right) on Fashion-MNIST dataset.}
    \label{fig:datailed_scores_fashion}
\end{figure*}

Next, we analyze which samples from the nominal class are localized close to or furthest from the center of bounding hypersphere. It is evident from Figure \ref{fig:samples} that \our{} maps images with regular structure, which are easy to recognize, in the hypersphere center. On the other hand, examples localized far from the center (outside the bounding region) do not look visually plausible and one cannot be sure about their class. It means that \our{} gives results consistent with our intuition.

To give a better insight, we calculate AUC for every pair of classes, see Figures \ref{fig:datailed_scores_mnist} and \ref{fig:datailed_scores_fashion}. More precisely, every entry of the heatmap shows AUC obtained for a given nominal class listed in a column and a given anomaly class listed in a row. For example, it occurs that \our{} and \ourdet{} have the biggest problems with detecting anomalies represented by class "1" when trained on class "8". Generally, it is evident that the class "1" is the hardest to describe by our model. It may be explained by the fact that handwritten digit "1" can be written in various styles, which makes it similar to other classes.

\subsection{PIDForest benchmark} 

We make additional benchmarks following the experimental setting of ~\cite{gopalan2019pidforest}. More specifically, we test \our{} and \ourdet{} on the following eight datasets from the  UCI~\cite{asuncion2007uci}, openML repository~\cite{vanschoren2014openml} and KDD Cup 1999: Thyroid, Mammography, Seismic, Satimage-2, Vowels, Musk, http, smtp. 

For a comparison, we use the following methods: PIDForest \cite{gopalan2019pidforest}, IsolationForest (iForest) \cite{liu2012isolation}, Robust Random Cut Forest (RRCF) \cite{guha2016robust}, Local Outlier Factor (LOF) \cite{breunig2000lof}, k-Nearest Neighbour (kNN), Principal Component Analysis (PCA) \cite{wold1987principal}. Every method was trained on the whole dataset (outliers included) and evaluation measure was reported on the same set. To reproduce the results from~\cite{gopalan2019pidforest}, we use AUC as the evaluation measure. Figure ~\ref{fig:pid_ranks} (left) presents box plot of ranks, which summarizes the experiment (Table ~\ref{tab:pid_roc} from Appendix~\ref{app:detailed_results} contains detailed scores).
While the performance of \our{} and \ourdet{} in overall is slightly worse than the best algorithms (PIDForest and iForest), it outperforms these methods on some of the datasets (Satimage, Vowels, http).

\begin{figure*}[t] 
    \centering
    \subfigure{%
  	    \centering
        \includegraphics[width=0.49\textwidth]{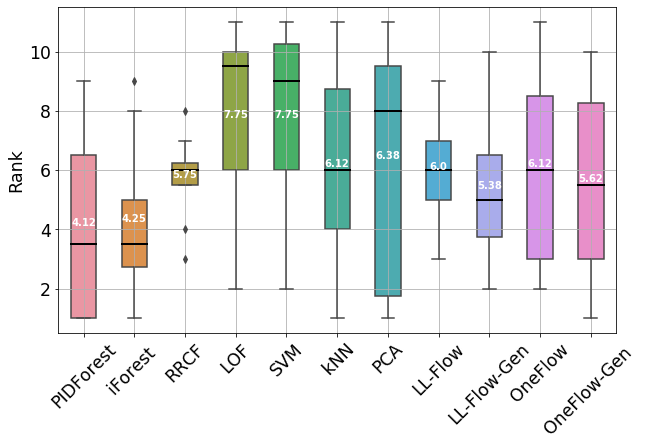}
        \label{fig:pid_ranks_roc}
      } 
    \subfigure{%
      	\centering
        \includegraphics[width=0.49\textwidth]{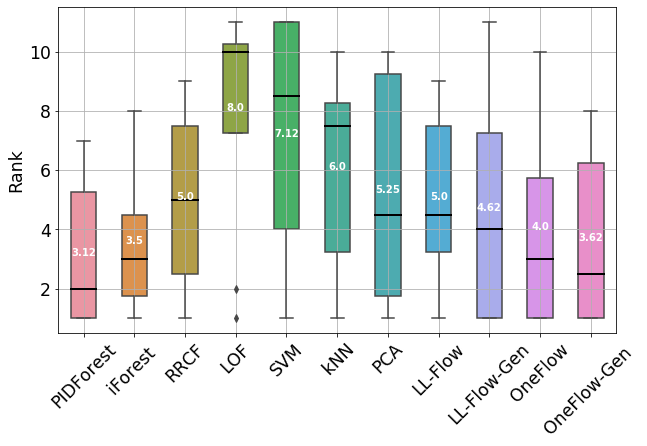}
        \label{fig:pid_ranks_f1}
      } 
    \caption{Box plots of ranks calculated using AUC (left) and F1 (right) on datasets from PIDForest benchmark. The median ranking is marked by a line, while the average ranking is marked with a number. The results of comparative methods (except LL-Flow and LL-Flow-Gen) are taken from \cite{gopalan2019pidforest}.}
    \label{fig:pid_ranks}
\end{figure*}

Let us recall that AUC does not take into account a single decision boundary, but effectively considers all possible test set thresholds. In the production environment, we need to have a classification rule and verify which algorithm detects outliers and nominal data correctly. For this reason, we also test all algorithms in a discriminative setting, in which 5\% of farthest examples (according to a given loss function) are deemed as anomalies. The results are evaluated using the F1 score.  We present rank plot in Figure~\ref{fig:pid_ranks} (right) (detailed scores are reported in Appendix~\ref{app:detailed_results} in Table~\ref{tab:pid_f1}).
One can see, that the performance of \our{} is comparable to the best algorithms (again PIDForest and iForest). \our{} or \ourdet{} is better than PIDForest and iForest in 4 out of 8 datasets, performs the same in 2 out of 8 datasets and is worse in 2 out of 8 datasets.
This experiment confirms that the proposed method is better at finding outliers, which is not the same as ranking elements according to the loss function, but is crucial in practice.

\begin{figure*}[t] 
    \centering
        \includegraphics[width=0.48\textwidth]{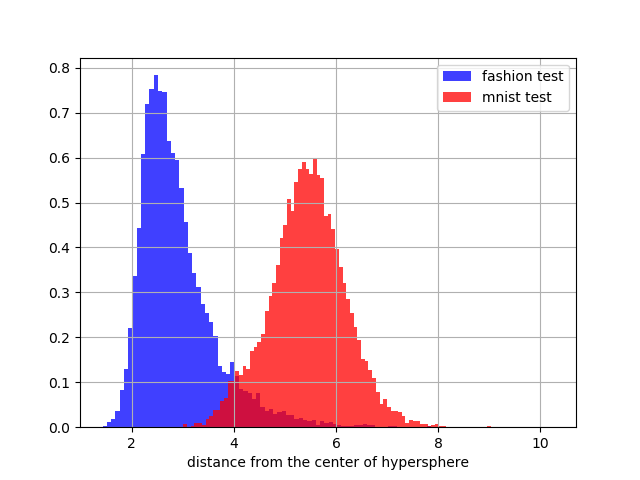}
        \includegraphics[width=0.48\textwidth]{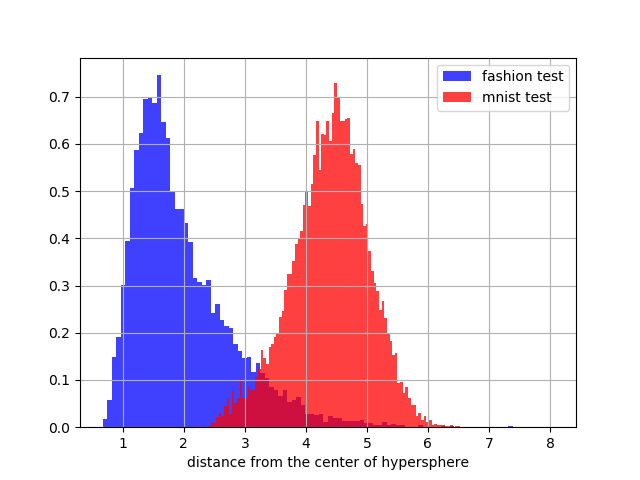}
    \caption{Distance of Fashion-MNIST nominal data (blue) and MNIST (representing outliers) data (red) from the center of bounding hypersphere in the latent space of \our{} (left) and \flow{} (right). The percentage of detected outliers equals $92.47\%$ for \our{} and $89.59\%$ for \flow{}.}
    \label{fig:ood_fmnist}
\end{figure*}

\subsection{Test on Out-Of-Distribution dataset} 
We now focus on comparing \our{} with \flow{}, which represents its natural baseline, and follow the experiment recently suggested in \cite{nalisnick2018deep}. In this setting, each model is trained on the Fashion-MNIST train set with $\alpha =5\%$. Next, we test these models on the data coming from test sets of both MNIST and Fashion-MNIST. We expect that the models will be able to classify MNIST examples as anomalies.

Figure \ref{fig:ood_fmnist} illustrates the distance of latent representations from the center of bounding hypersphere. As expected, in both cases, Fashion-MNIST (nominal) data are localized closer to the center than MNIST (outliers) data, which is correct behavior. However, \our{} maps out-of-distribution data (MNIST) much further from the center than \flow{} (see the range on x-axis). We verified  that the percentage of correctly classified anomalies equals $92.47\%$ for \our{} and $89.59\%$ for \flow{}, which means that the discriminative power of \our{} is higher than capabilities of \flow{}. We also performed the second experiment when MNIST was considered as nominal data and Fashion-MNIST represented outliers and both models obtained almost perfect performance in this situation.


\subsection{Illustrative examples} 
To find key differences between \our{} and \flow{}, we consider 2-dimensional examples, which are easy to visualize and represent a typical benchmark for comparing anomaly detection algorithms (additional illustrative examples on 3D point clouds are presented in Figure~\ref{fig:plots_3d_car1_1}). 

\begin{figure*}[t] 
    \centering
    \includegraphics[width=0.49\textwidth]{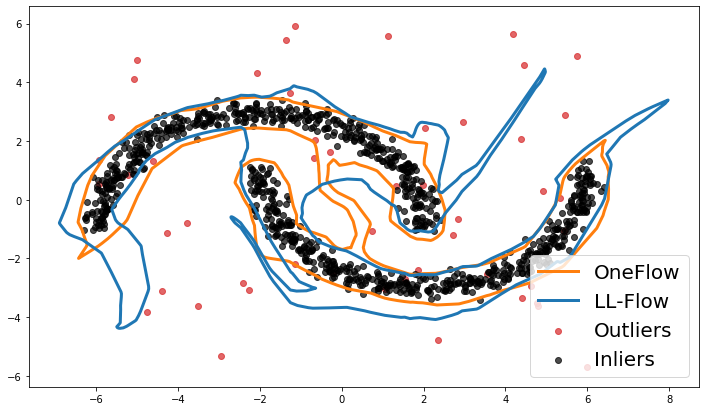}
    \includegraphics[width=0.49\textwidth]{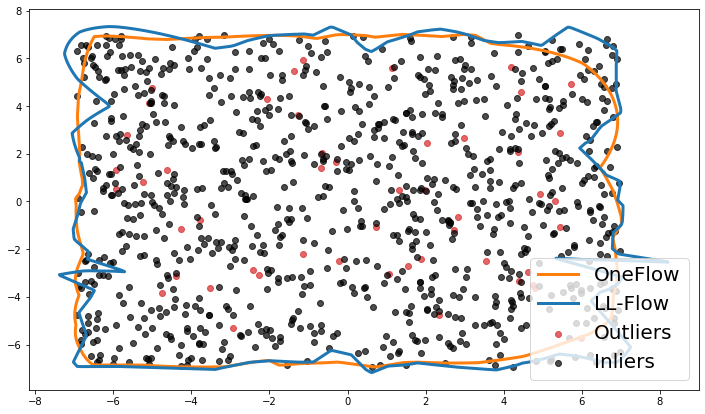}\\
    \includegraphics[width=0.49\textwidth]{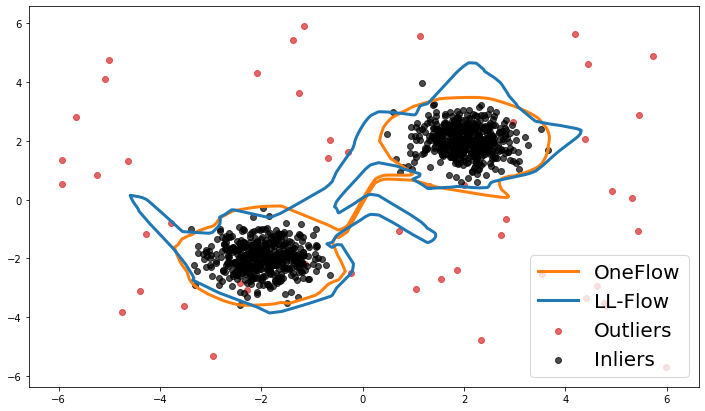}
    \includegraphics[width=0.49\textwidth]{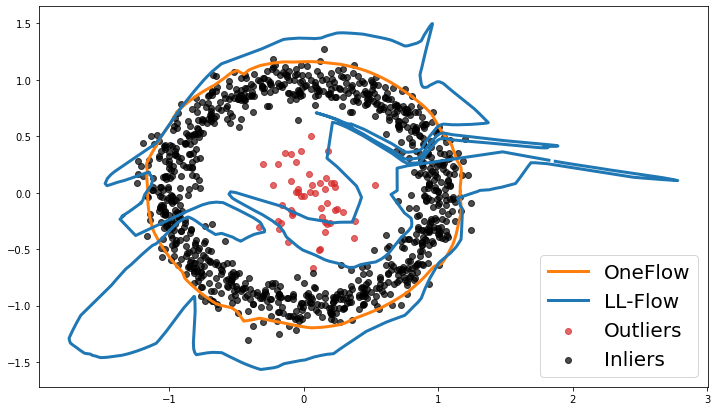}
    \caption{Comparing bounding regions constructed by the proposed \our{} with a flow-based density model (\flow{}).}
    \label{fig:plots_2d}
\end{figure*}

At first glance, both flow models give similar results in most cases, see Figure \ref{fig:plots_2d}. However, a closer inspection reveals that the decision boundaries created by \our{} are smoother and shorter than the ones resulted from \flow{}. While minimizing the area of the bounding region should lead to a more accurate description of the nominal class, minimizing the length of the decision boundary reduces the model complexity. Making an analogy with typical supervised models, smooth, short, and simple decision boundaries usually increases the generalization performance of the model to unseen examples. To confirm this observation, we calculate the volume of the bounding region and the length of the corresponding decision boundary, see Table \ref{tab:r}. It is evident that these quantities are smaller in the case of \our{}. 

A notable difference between both models can be seen in a doughnut shape distribution. \flow{} is trying to remove anomalies from the centre of the doughnut, however is too rigid, which causes the long and sharp decision boundary. On the other hand, \our{} optimizes the length of the decision boundary and treats points on the doughnut's outer edge as outliers. This behavior is typical for OneFlow’s logic, because our model is not interested in density estimation.

\begin{table}[t]
\caption{Comparing the length of decision boundaries and the volume of bounding regions constructed by \our{} and \flow{} for 2D examples (lower is better).}\label{tab:r}
\centering
\begin{tabular}{l l l l l}
\toprule
\multirow{2}{*}{Dataset} & \multicolumn{2}{c}{Length} & \multicolumn{2}{c}{Volume}\\
 & \our{} & \flow{} & \our{} & \flow{} \\
\midrule
Two Moons & 65.401 & 84.015 & 31.410 & 41.533 \\
Big Uniform & 52.015 & 68.076 & 182.702 & 184.431 \\
Two Blobs & 26.739 & 36.863 & 16.749 & 21.058 \\
Doughnut & 7.482 & 27.205 & 4.316 & 5.435 \\
Diverse Blobs & 19.754 & 26.685 & 25.661 & 25.823 \\
\bottomrule
\end{tabular}
\end{table}

Another example, in which the behavior of \our{} and \flow{} is different, is shown in Figure \ref{fig:sat}, where a dataset consists of two diverse blobs -- the one with 98\% of data and the second with remaining 2\% of data. While \flow{} focuses on the whole data and considers a few examples from the smaller blob as nominal data, \our{} directly solves a one-class problem and deems the whole smaller blob as anomalies. It shows that \our{} is not very sensitive to the structure and the distribution of anomalies, because they are automatically ignored in a training phase. On the other hand, \flow{} fits a prior density to the whole data, and, in consequence, a distribution of anomalies has an influence on the final results. Analogical behavior of both models was observed when we changed the proportions of clusters. Indeed, \our{} always deemed smaller cluster as anomalies if it contains at most 5\% of the whole data (larger clusters cannot be considered as anomalies in practice). In contrast, \flow{} could not separate the small clusters from nominal data in this case.

\subsection{Analysis of parameter $\alpha$} 

\begin{figure*}[t]
  \centering
  \includegraphics[width=0.48\textwidth]{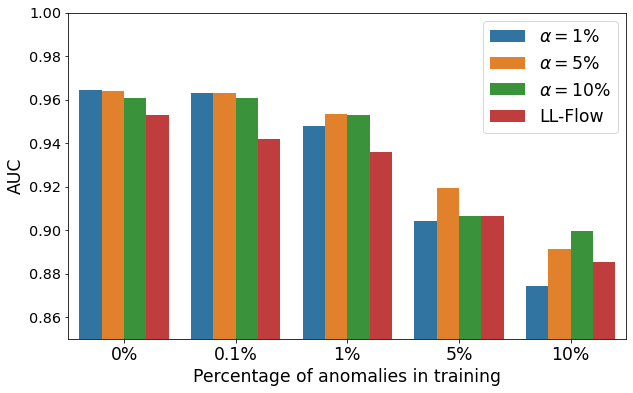}  \includegraphics[width=0.48\textwidth]{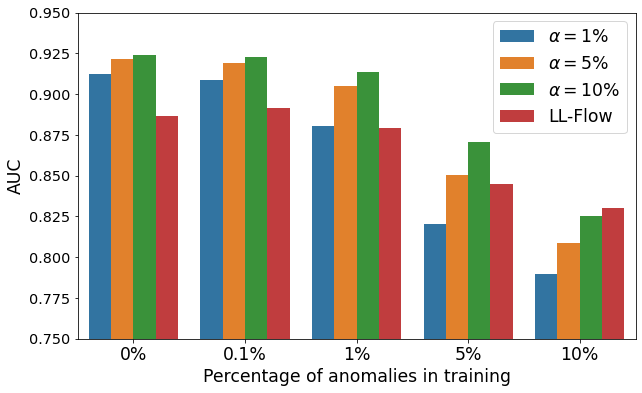}
  \caption{Evaluation of \our{} with three levels of $\alpha=1\%,5\%,10\%$ on MNIST (left) and Fashion-MNIST (right), in which nominal class is corrupted with anomalies at training time.} \label{fig:noisymnist}
\end{figure*}

Previous experiments were performed for \our{} with $\alpha=5\%$. A natural question is: what is the influence of $\alpha$ on the behavior of \our{}?

To partially answer this question, we corrupt a training nominal data with anomalies. We consider 5 noise levels with: 0\%, 0.1\%, 1\%, 5\% and 10\% of anomalies in a training set. In each case, we run \our{} with $\alpha = 1\%, 5\%, 10\%$. Evaluation on test set remains exactly the same as before. We take into account MNIST and Fashion-MNIST datasets. For a comparison, we also included LL-Flow.

It is clear from Figure \ref{fig:noisymnist} that the performance of \our{} slightly deteriorates as the number of anomalies in training set increases regardless of the value of $\alpha$. Moreover, the model with a high value of $\alpha$ is able to deal with a large number of anomalies in training better than the model with small $\alpha$. Indeed, \our{} with $\alpha=10\%$ leaves approximately $10\%$ of data outside the bounding region and thus can still provide a good description of nominal data as long as the number of anomalies does not exceed $10\%$. Moreover on almost all settings \our{} outperforms LL-Flow, being worse only in case of Fashion-MNIST with 10\% anomalies in the training dataset.

Another observation is that \our{} with small $\alpha$ works better for MNIST than for Fashion-MNIST when the number of anomalies is low (less than $1\%$). It may be explained by the fact that MNIST is a relatively simple dataset and almost all examples from each class are similar. In consequence, \our{} with $\alpha=1\%$ performs better than with $\alpha=10\%$ for negligible amount of anomalies in training. On the other hand, the variation in each class of Fashion-MNIST is greater (in the noiseless case, AUC for Fashion-MNIST is 4 percentage points lower than for MNIST) so the bounding region created for $\alpha = 1\%$ is too loose. In the test phase, such a bounding region may contain too many anomalies. 

The above analysis suggests that $\alpha$ should be large if the underlying anomaly detection task is hard or we have many anomalies in training. Otherwise, we should keep $\alpha$ low.

\subsection{Analysis of training epochs number}

\begin{figure*}[h!]
  \centering
  \includegraphics[width=0.48\textwidth]{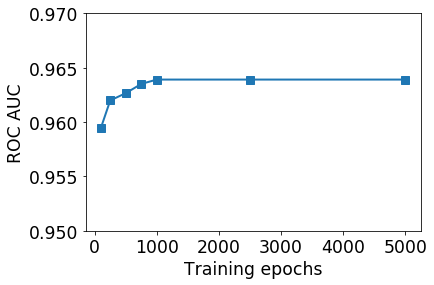}
  \includegraphics[width=0.48\textwidth]{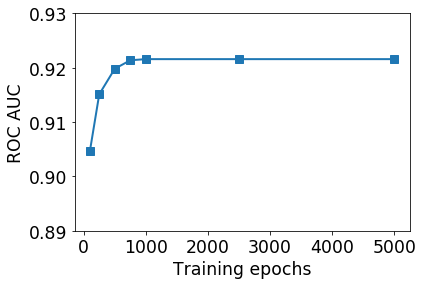}
  \caption{Evaluation of \our{} on MNIST (left) and Fashion-MNIST (right) with seven different numbers of training epochs: 100, 250, 500, 750, 1000, 2500, 5000.} \label{fig:epochsmnist}
\end{figure*}

In the following subsection we show scores obtained by \our{} using different number of training epochs. It is evident from Figure~\ref{fig:epochsmnist} that 1000 epochs that we used in our benchmarks provides good trade off between training time and obtained performance. Indeed for both MNIST as well as Fashion-MNIST datasets obtained ROC AUC grows until it reaches 1000 epochs and then flats out.

\subsection{Analysis of backbone model}

In our experiments we used a Nice-based~\cite{dinh2014nice} architecture as our backbone flow model. In order to check whether the change of backbone would improve the \our{} results, we replaced Nice with Real-NVP~\cite{dinh2016density} and Glow~\cite{kingma2018glow} architectures, and tested the obtained models on MNIST and Fashion-MNIST datasets.

The results presented in Table~\ref{tab:backbone} show that changing the Nice architecture to the more sophisticated ones significantly worsens the performance of \our{}. This could be caused by the fact that these models are too complex for the task of outlier detection on MNIST and Fashion-MNIST. Indeed, \our{} trained with Nice backbone obtained results close to the state-of-the-art on this task. Considering this and the fact that Real-NVP and Glow are designed to operate on image data, we decided to use Nice as a default backbone model in \our{}. More detailed results could be found in Appendix~\ref{app:detailed_results} in Tables~\ref{tab:mnist} and \ref{tab:fashion} (models are signed as OF-RNVP and OF-Glow).

\begin{table}[h]
\caption{Comparison of the flow backbone model for \our{}.
Real-NVP and Glow are more sophisticated flow architectures designed to work on images.}\label{tab:backbone}
\centering
\begin{tabular}{l c c}
\toprule
{} & MNIST & Fashion-MNIST \\
\midrule
\our{}   & .964 & .922 \\
\midrule
Real-NVP & .683 & .769 \\
Glow     & .895 & .775 \\
\bottomrule
\end{tabular}
\end{table}

\section{Conclusion}

The paper introduced \our{}, which realizes a well-known one-class paradigm using deep learning tools. Making use of a flow-based model and a Bernstein quantile estimator, we find a minimal volume bounding region for a given percentage of data. On the one hand, the constructed bounding region does not depend on the structure of outliers as in the density-based models, while, on the other hand, the bounding region is given by an explicit parametric form. Experimental results demonstrate that \our{} presents state-of-the-art performance.

\section*{Acknowledgments}
The work of L. Maziarka was supported by the National Science Centre (Poland) grant no. 2018/31/B/ST6/00993. The work of M. Śmieja, Ł. Struski were supported by the Foundation for Polish Science co-financed by the European Union under the European Regional Development Fund in the POIR.04.04.00-00-14DE/18-00 project is carried out within the Team-Net programme. The work of J. Tabor was supported by the National Science Centre (Poland) Grant No. 2017/25/B/ST6/01271. The work of P. Spurek was supported by the National Science Centre (Poland) Grant No. 2019/33/B/ST6/00894. 

\ifCLASSOPTIONcaptionsoff
  \newpage
\fi

\bibliographystyle{IEEEtran}
\bibliography{paper}

\begin{thebibliography}{10}
\providecommand{\url}[1]{#1}
\csname url@samestyle\endcsname
\providecommand{\newblock}{\relax}
\providecommand{\bibinfo}[2]{#2}
\providecommand{\BIBentrySTDinterwordspacing}{\spaceskip=0pt\relax}
\providecommand{\BIBentryALTinterwordstretchfactor}{4}
\providecommand{\BIBentryALTinterwordspacing}{\spaceskip=\fontdimen2\font plus
\BIBentryALTinterwordstretchfactor\fontdimen3\font minus
  \fontdimen4\font\relax}
\providecommand{\BIBforeignlanguage}[2]{{%
\expandafter\ifx\csname l@#1\endcsname\relax
\typeout{** WARNING: IEEEtran.bst: No hyphenation pattern has been}%
\typeout{** loaded for the language `#1'. Using the pattern for}%
\typeout{** the default language instead.}%
\else
\language=\csname l@#1\endcsname
\fi
#2}}
\providecommand{\BIBdecl}{\relax}
\BIBdecl

\bibitem{miljkovic2010review}
D.~Miljkovi{\'c}, ``Review of novelty detection methods,'' in \emph{The 33rd
  International Convention MIPRO}.\hskip 1em plus 0.5em minus 0.4em\relax IEEE,
  2010, pp. 593--598.

\bibitem{phua2010comprehensive}
C.~Phua, V.~Lee, K.~Smith, and R.~Gayler, ``A comprehensive survey of data
  mining-based fraud detection research,'' \emph{arXiv preprint
  arXiv:1009.6119}, 2010.

\bibitem{lavin2015evaluating}
A.~Lavin and S.~Ahmad, ``Evaluating real-time anomaly detection algorithms--the
  numenta anomaly benchmark,'' in \emph{2015 IEEE 14th International Conference
  on Machine Learning and Applications (ICMLA)}.\hskip 1em plus 0.5em minus
  0.4em\relax IEEE, 2015, pp. 38--44.

\bibitem{roth2019odds}
K.~Roth, Y.~Kilcher, and T.~Hofmann, ``The odds are odd: A statistical test for
  detecting adversarial examples,'' \emph{arXiv preprint arXiv:1902.04818},
  2019.

\bibitem{garcia2009anomaly}
P.~Garcia-Teodoro, J.~Diaz-Verdejo, G.~Maci{\'a}-Fern{\'a}ndez, and
  E.~V{\'a}zquez, ``Anomaly-based network intrusion detection: Techniques,
  systems and challenges,'' \emph{computers \& security}, vol.~28, no. 1-2, pp.
  18--28, 2009.

\bibitem{shone2018deep}
N.~Shone, T.~N. Ngoc, V.~D. Phai, and Q.~Shi, ``A deep learning approach to
  network intrusion detection,'' \emph{IEEE Transactions on Emerging Topics in
  Computational Intelligence}, vol.~2, no.~1, pp. 41--50, 2018.

\bibitem{goh2017anomaly}
J.~Goh, S.~Adepu, M.~Tan, and Z.~S. Lee, ``Anomaly detection in cyber physical
  systems using recurrent neural networks,'' in \emph{2017 IEEE 18th
  International Symposium on High Assurance Systems Engineering (HASE)}.\hskip
  1em plus 0.5em minus 0.4em\relax IEEE, 2017, pp. 140--145.

\bibitem{scholkopf2001estimating}
B.~Sch{\"o}lkopf, J.~C. Platt, J.~Shawe-Taylor, A.~J. Smola, and R.~C.
  Williamson, ``Estimating the support of a high-dimensional distribution,''
  \emph{Neural computation}, vol.~13, no.~7, pp. 1443--1471, 2001.

\bibitem{abati2019latent}
D.~Abati, A.~Porrello, S.~Calderara, and R.~Cucchiara, ``Latent space
  autoregression for novelty detection,'' in \emph{Proceedings of the IEEE
  Conference on Computer Vision and Pattern Recognition}, 2019, pp. 481--490.

\bibitem{li2018anomaly}
D.~Li, D.~Chen, J.~Goh, and S.-K. Ng, ``Anomaly detection with generative
  adversarial networks for multivariate time series,'' \emph{arXiv preprint
  arXiv:1809.04758v3}, 2019.

\bibitem{wang2019effective}
S.~Wang, Y.~Zeng, X.~Liu, E.~Zhu, J.~Yin, C.~Xu, and M.~Kloft, ``Effective
  end-to-end unsupervised outlier detection via inlier priority of
  discriminative network,'' in \emph{Advances in Neural Information Processing
  Systems}, 2019, pp. 5960--5973.

\bibitem{scott2006learning}
C.~D. Scott and R.~D. Nowak, ``Learning minimum volume sets,'' \emph{Journal of
  Machine Learning Research}, vol.~7, no. Apr, pp. 665--704, 2006.

\bibitem{zhao2009anomaly}
M.~Zhao and V.~Saligrama, ``Anomaly detection with score functions based on
  nearest neighbor graphs,'' in \emph{Advances in neural information processing
  systems}, 2009, pp. 2250--2258.

\bibitem{tax2004support}
D.~M. Tax and R.~P. Duin, ``Support vector data description,'' \emph{Machine
  learning}, vol.~54, no.~1, pp. 45--66, 2004.

\bibitem{dinh2014nice}
L.~Dinh, D.~Krueger, and Y.~Bengio, ``Nice: Non-linear independent components
  estimation,'' \emph{arXiv preprint arXiv:1410.8516}, 2014.

\bibitem{kingma2018glow}
D.~P. Kingma and P.~Dhariwal, ``Glow: Generative flow with invertible 1x1
  convolutions,'' in \emph{Advances in Neural Information Processing Systems},
  2018, pp. 10\,215--10\,224.

\bibitem{yang2019pointflow}
G.~Yang, X.~Huang, Z.~Hao, M.-Y. Liu, S.~Belongie, and B.~Hariharan,
  ``Pointflow: 3d point cloud generation with continuous normalizing flows,''
  in \emph{Proceedings of the IEEE International Conference on Computer
  Vision}, 2019, pp. 4541--4550.

\bibitem{spurek2020hypernetwork}
P.~Spurek, S.~Winczowski, J.~Tabor, M.~Zamorski, M.~Zieba, and T.~Trzcinski,
  ``Hypernetwork approach to generating point clouds,'' in \emph{International
  Conference on Machine Learning}.\hskip 1em plus 0.5em minus 0.4em\relax PMLR,
  2020, pp. 9099--9108.

\bibitem{cheng1995bernstein}
C.~Cheng, ``The bernstein polynomial estimator of a smooth quantile function,''
  \emph{Statistics \& probability letters}, vol.~24, no.~4, pp. 321--330, 1995.

\bibitem{asuncion2007uci}
A.~Asuncion and D.~Newman, ``Uci machine learning repository,'' 2007.

\bibitem{vanschoren2014openml}
J.~Vanschoren, J.~N. Van~Rijn, B.~Bischl, and L.~Torgo, ``Openml: networked
  science in machine learning,'' \emph{ACM SIGKDD Explorations Newsletter},
  vol.~15, no.~2, pp. 49--60, 2014.

\bibitem{chen2013new}
Y.~Chen, J.~Qian, and V.~Saligrama, ``A new one-class svm for anomaly
  detection,'' in \emph{2013 IEEE International Conference on Acoustics, Speech
  and Signal Processing}.\hskip 1em plus 0.5em minus 0.4em\relax IEEE, 2013,
  pp. 3567--3571.

\bibitem{ruff2018deep}
L.~Ruff, R.~Vandermeulen, N.~Goernitz, L.~Deecke, S.~A. Siddiqui, A.~Binder,
  E.~M{\"u}ller, and M.~Kloft, ``Deep one-class classification,'' in
  \emph{International conference on machine learning}, 2018, pp. 4393--4402.

\bibitem{kim2015deep}
S.~Kim, Y.~Choi, and M.~Lee, ``Deep learning with support vector data
  description,'' \emph{Neurocomputing}, vol. 165, pp. 111--117, 2015.

\bibitem{ruff2019deep}
L.~Ruff, R.~A. Vandermeulen, N.~G{\"o}rnitz, A.~Binder, E.~M{\"u}ller, K.-R.
  M{\"u}ller, and M.~Kloft, ``Deep semi-supervised anomaly detection,''
  \emph{arXiv preprint arXiv:1906.02694}, 2019.

\bibitem{chong2020simple}
P.~Chong, L.~Ruff, M.~Kloft, and A.~Binder, ``Simple and effective prevention
  of mode collapse in deep one-class classification,'' \emph{arXiv preprint
  arXiv:2001.08873}, 2020.

\bibitem{dasgupta2018neural}
S.~Dasgupta, T.~C. Sheehan, C.~F. Stevens, and S.~Navlakha, ``A neural data
  structure for novelty detection,'' \emph{Proceedings of the National Academy
  of Sciences}, vol. 115, no.~51, pp. 13\,093--13\,098, 2018.

\bibitem{shu2018unseen}
L.~Shu, H.~Xu, and B.~Liu, ``Unseen class discovery in open-world
  classification,'' \emph{arXiv preprint arXiv:1801.05609}, 2018.

\bibitem{schlegl2017unsupervised}
T.~Schlegl, P.~Seeb{\"o}ck, S.~M. Waldstein, U.~Schmidt-Erfurth, and G.~Langs,
  ``Unsupervised anomaly detection with generative adversarial networks to
  guide marker discovery,'' in \emph{International conference on information
  processing in medical imaging}.\hskip 1em plus 0.5em minus 0.4em\relax
  Springer, 2017, pp. 146--157.

\bibitem{schlegl2019f}
T.~Schlegl, P.~Seeb{\"o}ck, S.~M. Waldstein, G.~Langs, and U.~Schmidt-Erfurth,
  ``f-anogan: Fast unsupervised anomaly detection with generative adversarial
  networks,'' \emph{Medical image analysis}, vol.~54, pp. 30--44, 2019.

\bibitem{deecke2018image}
L.~Deecke, R.~Vandermeulen, L.~Ruff, S.~Mandt, and M.~Kloft, ``Image anomaly
  detection with generative adversarial networks,'' in \emph{Joint european
  conference on machine learning and knowledge discovery in databases}.\hskip
  1em plus 0.5em minus 0.4em\relax Springer, 2018, pp. 3--17.

\bibitem{perera2019ocgan}
P.~Perera, R.~Nallapati, and B.~Xiang, ``Ocgan: One-class novelty detection
  using gans with constrained latent representations,'' in \emph{Proceedings of
  the IEEE Conference on Computer Vision and Pattern Recognition}, 2019, pp.
  2898--2906.

\bibitem{sabokrou2018adversarially}
M.~Sabokrou, M.~Khalooei, M.~Fathy, and E.~Adeli, ``Adversarially learned
  one-class classifier for novelty detection,'' in \emph{Proceedings of the
  IEEE Conference on Computer Vision and Pattern Recognition}, 2018, pp.
  3379--3388.

\bibitem{hendrycks2016baseline}
D.~Hendrycks and K.~Gimpel, ``A baseline for detecting misclassified and
  out-of-distribution examples in neural networks,'' \emph{arXiv preprint
  arXiv:1610.02136}, 2016.

\bibitem{liang2017enhancing}
S.~Liang, Y.~Li, and R.~Srikant, ``Enhancing the reliability of
  out-of-distribution image detection in neural networks,'' \emph{arXiv
  preprint arXiv:1706.02690}, 2017.

\bibitem{devries2018learning}
T.~DeVries and G.~W. Taylor, ``Learning confidence for out-of-distribution
  detection in neural networks,'' \emph{arXiv preprint arXiv:1802.04865}, 2018.

\bibitem{wang2019multivariate}
J.~Wang, S.~Sun, and Y.~Yu, ``Multivariate triangular quantile maps for novelty
  detection,'' in \emph{Advances in Neural Information Processing Systems},
  2019, pp. 5061--5072.

\bibitem{schmidt2019normalizing}
M.~Schmidt and M.~Simic, ``Normalizing flows for novelty detection in
  industrial time series data,'' \emph{arXiv preprint arXiv:1906.06904}, 2019.

\bibitem{nachman2020anomaly}
B.~Nachman and D.~Shih, ``Anomaly detection with density estimation,''
  \emph{Physical Review D}, vol. 101, no.~7, p. 075042, 2020.

\bibitem{wellhausen2020safe}
L.~Wellhausen, R.~Ranftl, and M.~Hutter, ``Safe robot navigation via
  multi-modal anomaly detection,'' \emph{IEEE Robotics and Automation Letters},
  vol.~5, no.~2, pp. 1326--1333, 2020.

\bibitem{ruff2021unifying}
L.~Ruff, J.~R. Kauffmann, R.~A. Vandermeulen, G.~Montavon, W.~Samek, M.~Kloft,
  T.~G. Dietterich, and K.-R. M{\"u}ller, ``A unifying review of deep and
  shallow anomaly detection,'' \emph{Proceedings of the IEEE}, 2021.

\bibitem{pang2020deep}
G.~Pang, C.~Shen, L.~Cao, and A.~v.~d. Hengel, ``Deep learning for anomaly
  detection: A review,'' \emph{arXiv preprint arXiv:2007.02500}, 2020.

\bibitem{dinh2016density}
L.~Dinh, J.~Sohl-Dickstein, and S.~Bengio, ``Density estimation using real
  nvp,'' \emph{arXiv preprint arXiv:1605.08803}, 2016.

\bibitem{leblanc2012estimating}
A.~Leblanc, ``On estimating distribution functions using bernstein
  polynomials,'' \emph{Annals of the Institute of Statistical Mathematics},
  vol.~64, no.~5, pp. 919--943, 2012.

\bibitem{zielinski2004optimal}
R.~Zielinski, \emph{Optimal Quantile Estimators Small Sample Approach}.\hskip
  1em plus 0.5em minus 0.4em\relax Polish Academy of Sciences. Institute of
  Mathematics, 2004.

\bibitem{zhai2016deep}
S.~Zhai, Y.~Cheng, W.~Lu, and Z.~Zhang, ``Deep structured energy based models
  for anomaly detection,'' \emph{arXiv preprint arXiv:1605.07717}, 2016.

\bibitem{zong2018deep}
B.~Zong, Q.~Song, M.~R. Min, W.~Cheng, C.~Lumezanu, D.~Cho, and H.~Chen, ``Deep
  autoencoding gaussian mixture model for unsupervised anomaly detection,''
  2018.

\bibitem{golan2018deep}
I.~Golan and R.~El-Yaniv, ``Deep anomaly detection using geometric
  transformations,'' in \emph{Advances in Neural Information Processing
  Systems}, 2018, pp. 9758--9769.

\bibitem{kingma2013auto}
D.~P. Kingma and M.~Welling, ``Auto-encoding variational bayes,'' \emph{arXiv
  preprint arXiv:1312.6114}, 2013.

\bibitem{vincent2008extracting}
P.~Vincent, H.~Larochelle, Y.~Bengio, and P.-A. Manzagol, ``Extracting and
  composing robust features with denoising autoencoders,'' in \emph{Proceedings
  of the 25th international conference on Machine learning}, 2008, pp.
  1096--1103.

\bibitem{pidhorskyi2018generative}
S.~Pidhorskyi, R.~Almohsen, and G.~Doretto, ``Generative probabilistic novelty
  detection with adversarial autoencoders,'' in \emph{Advances in neural
  information processing systems}, 2018, pp. 6822--6833.

\bibitem{gopalan2019pidforest}
P.~Gopalan, V.~Sharan, and U.~Wieder, ``Pidforest: Anomaly detection via
  partial identification,'' in \emph{Advances in Neural Information Processing
  Systems}, 2019, pp. 15\,809--15\,819.

\bibitem{liu2012isolation}
F.~T. Liu, K.~M. Ting, and Z.-H. Zhou, ``Isolation-based anomaly detection,''
  \emph{ACM Transactions on Knowledge Discovery from Data (TKDD)}, vol.~6,
  no.~1, pp. 1--39, 2012.

\bibitem{guha2016robust}
S.~Guha, N.~Mishra, G.~Roy, and O.~Schrijvers, ``Robust random cut forest based
  anomaly detection on streams,'' in \emph{International conference on machine
  learning}, 2016, pp. 2712--2721.

\bibitem{breunig2000lof}
M.~M. Breunig, H.-P. Kriegel, R.~T. Ng, and J.~Sander, ``Lof: identifying
  density-based local outliers,'' in \emph{Proceedings of the 2000 ACM SIGMOD
  international conference on Management of data}, 2000, pp. 93--104.

\bibitem{wold1987principal}
S.~Wold, K.~Esbensen, and P.~Geladi, ``Principal component analysis,''
  \emph{Chemometrics and intelligent laboratory systems}, vol.~2, no. 1-3, pp.
  37--52, 1987.

\bibitem{nalisnick2018deep}
E.~Nalisnick, A.~Matsukawa, Y.~W. Teh, D.~Gorur, and B.~Lakshminarayanan, ``Do
  deep generative models know what they don't know?'' \emph{arXiv preprint
  arXiv:1810.09136}, 2018.

\end{thebibliography}

\appendices

\section{Notations} \label{app:notation}

\begin{itemize}
    \item $g$ -- a density in $\R^D$ producing data,
    \item $X \subset \R^D$ -- a sample generated by density $g$,
    \item $(x_1,\ldots,x_n) \subset X$ -- a mini-batch,
    \item $\alpha \in (0,1)$ -- a $p$-value, which determines the percentage of possible outliers in a dataset,
    \item $U \subset \R^D$ -- a $(1-\alpha)$-bounding region (i.e. $\int_U g(x) \, dx=1-\alpha$),
    \item $s_\alpha$ -- the Bernstein estimation of $( 1 - \alpha )$-quantile based on a sample $s_1,
    \ldots,s_n$.
    \item $f: \R^D \to \R^D$ -- a transformation given by a neural network.
    \item $f_\theta: \R^D \to \R^D$ -- a transformation given by a neural network with weights $\theta$.
    \item $\det d f(x)$ -- the Jacobian determinant of neural network $f$.
    \item $B(0,r)$ -- a ball centered at $0$ with radius $r$.
\end{itemize}

\section{Experimental setting} \label{app:exp}

In all our experiments, OneFlow and LL-Flow are implemented using the architecture of NICE flow model~\cite{dinh2014nice}, with the following hyperparameters:

\subsection{2D datasets}
\begin{itemize}
    \item Number of flow layers: 4
    \item Number of coupling layers: 4
    \item Hidden dimension: 16
    \item Number of epochs: 1000
    \item Batch size: 1000
    \item Learning rate: 0.001
\end{itemize}

\subsection{PIDForest benchmark}
\begin{itemize}
    \item Number of flow layers: 2
    \item Number of coupling layers: 6
    \item Hidden dimension: 64
    \item Number of epochs: 2000
    \item Batch size: 1000
    \item Learning rate: 0.001
\end{itemize}

\subsection{Other anomaly detection and image datasets}
\begin{itemize}
    \item Number of flow layers: 4
    \item Number of coupling layers: 4
    \item Hidden dimension: 256
    \item Number of epochs: 1000
    \item Batch size: 1000
    \item Learning rate: 0.001
\end{itemize}

\section{Description of comparative algorithms} \label{app:desc}

Below, we give a brief description of algorithms used in the experimental section:
\begin{itemize}
    \item {\bf One-class SVM (OC-SVM)} \cite{scholkopf2001estimating}. It is a traditional kernel-based one-class classifier (we use the RBF kernel).
    \item {\bf Deep structured energy-based models (DSEBM)} \cite{zhai2016deep}. This model employs a deterministic deep neural network to output the energy function, such as negative log-likelihood, which is used to form the density of nominal data. 
    \item {\bf Deep autoencoding Gaussian mixture model (DAGMM)} \cite{zong2018deep}. It combines deep autoencoder with a Gaussian mixture estimation network to output the joint density of the latent representations and some reconstruction features from the autoencoder.
    \item {\bf Four variants of MQT -- multivariate quantile map (NLL, TQM$_1$, TQM$_2$, TQM$_\infty$)} \cite{wang2019multivariate}. MQT is a general model, which thresholds a given score function to describe nominal data. As a score function, we use negative log-likelihood (NLL) as well as 1-norm, 2-norm, and infinity norm of quantile (TQM).
    \item {\bf Deep Support Vector Data Description (DSVDD)} \cite{ruff2018deep}. It is an implementation of SVDD using deep neural networks, which penalizes data points that lie outside the hypersphere.
    \item {\bf Two variants of log-likelihood flow model (\flow{} and \flowdet{})}. These are generative flow models based on log-likelihood function that mimic \our{} and \ourdet{}, respectively.
    \item {\bf Geometric transformation (GT)} \cite{golan2018deep}. It uses a multi-class model to discriminate between dozens of geometric transformations applied to examples from the nominal class. The scoring function is the conditional probability of the softmax responses of the classifier given the geometric transformations.
    \item {\bf Variational autoencoder (VAE)} \cite{kingma2013auto}. The evidence lower bound is used as the scoring function.
    \item {\bf Denoising autoencoder (DAE)} \cite{vincent2008extracting}. The reconstruction error is used as the scoring function.
    \item {\bf Generative probabilistic novelty detection (GPND)} \cite{pidhorskyi2018generative}. GPND, based on adversarial autoencoders, uses data density as the scoring function. Density is approximated by linearizing the manifold that nominal data resides on.
    \item {\bf Latent space autoregression (LSA)} \cite{abati2019latent}. A parametric autoregressive model is used to estimate the density of the latent representation generated by a deep autoencoder. The sum of the normalized reconstruction error and log-likelihood is used as the scoring function.
    \item {\bf PIDForest} \cite{gopalan2019pidforest}. A random forest based algorithm that finds outliers based on the value of PIDScore, which is a geometric anomaly measure for a point. The scoring function measures the minimum density of data points over all subcubes containing the point.
    \item {\bf IsolationForest (iForest)} \cite{liu2012isolation}. A random forest based algorithm. iForest isolates observations by randomly selecting a feature and a split value. The number of splittings required to isolate a sample is the scoring function.
    \item {\bf Robust Random Cut Forest (RRCF)} \cite{guha2016robust}. An outlier detection algorithm that is based on a binary search tree. The scoring function is measured by its collusive displacement (CoDisp): if including a new point significantly changes the model complexity (i.e. bit depth), then that point is more likely to be an outlier.
    \item {\bf Local Outlier Factor (LOF)} \cite{breunig2000lof}. The scoring function is based on measuring the local deviation of a given data point with respect to its k-nearest neighbours.
    \item {\bf k-Nearest Neighbour (kNN)}. The distance from k-nearest neighbours is considered as the scoring function.
    \item {\bf Principal Component Analysis (PCA)} \cite{wold1987principal}. The scoring function is calculated as the distance from the axes in feature space.
\end{itemize}

\section{Detailed results for Experiments section} \label{app:detailed_results}

In Tables \ref{tab:pid_roc} and \ref{tab:pid_f1}, we present detailed results obtained for PIDForest benchmark.

\begin{table*}[tbh!]
\caption{AUC obtained on PIDForest benchmarks. The results of comparative methods (except LL-Flow and LL-Flow-Gen) are taken from \cite{gopalan2019pidforest}. Experiments were repeated $3$ times, we present the mean as well as standard deviation of obtained scores.} \label{tab:pid_roc}
\resizebox{\textwidth}{!}{%
\centering
\begin{tabular}{l | ccccccccccc }
\toprule
Data set & PIDForest & iForest  & RRCF & LOF & SVM &  kNN & PCA & LL-Flow & LL-Flow-Gen & OneFlow & OneFlow-Gen \\
\midrule
Thyroid &  .876 $\pm$ .013 & .819 $\pm$ .013 & .739 $\pm$ .004 & .737 & .547 & .751 & .673 & .848 $\pm$ .012 & .856 $\pm$ .012 & .605 $\pm$ .147 & .691 $\pm$ .078 \\ 
Mammo. &   .840 $\pm$ .010 & .862 $\pm$ .008 & .830 $\pm$ .002 & .720 & .872 & .839 & .886 & .856 $\pm$ .001 & .871 $\pm$ .007 & .710 $\pm$ .039 & .715 $\pm$ .057 \\ 
Seismic &  .733 $\pm$ .006 & .698 $\pm$ .004 & .701 $\pm$ .004 & .553 & .601 & .740 & .682 & .705 $\pm$ .007 & .708 $\pm$ .012 & .709 $\pm$ .004 & .719 $\pm$ .011  \\
Satimage & .987 $\pm$ .001 & .994 $\pm$ .001 & .991 $\pm$ .002 & .540 & .421 & .936 & .977 & .911 $\pm$ .007 & .982 $\pm$ .010 & .997 $\pm$ .000 & .998 $\pm$ .001 \\ 
Vowels &   .741 $\pm$ .008 & .736 $\pm$ .026 & .813 $\pm$ .007 & .943 & .778 & .975 & .606 & .837 $\pm$ .019 & .670 $\pm$ .012 & .846 $\pm$ .036 & .839 $\pm$ .056 \\ 
Musk &     1.00 $\pm$ .000 & .998 $\pm$ .003 & .998 $\pm$ .000 & .416 & .573 & .373 & 1.00 & .968 $\pm$ .005 & .989 $\pm$ .006 & .688 $\pm$ .011 & .946 $\pm$ .092 \\ 
http &     .986 $\pm$ .004 & 1.00 $\pm$ .000 & .993 $\pm$ .000 & .353 & .994 & .231 & .996 & .992 $\pm$ .001 & .991 $\pm$ .005 & .994 $\pm$ .000 & .994 $\pm$ .000 \\ 
smtp &     .923 $\pm$ .003 & .908 $\pm$ .003 & .886 $\pm$ .017 & .905 & .841 & .895 & .823 & .878 $\pm$ .016 & .900 $\pm$ .005 & .857 $\pm$ .016 & .841 $\pm$ .011 \\ 
\bottomrule
\end{tabular}
}
\end{table*}

\begin{table*}[tbh!]
\caption{F1 obtained on PIDForest benchmarks. Experiments were repeated $3$ times, we present the mean as well as standard deviation of obtained scores.} \label{tab:pid_f1}
\resizebox{\textwidth}{!}{%
\centering
\begin{tabular}{l | ccccccccccc }
\toprule
Data set & PIDForest & iForest  & RRCF & LOF & SVM &  kNN & PCA & LL-Flow & LL-Flow-Gen & OneFlow & OneFlow-Gen \\ 
\midrule
Thyroid  & .265 $\pm$ .054 & .311 $\pm$ .018 & .263 $\pm$ .008 & .242 & .103 & .255 & .226 & .299 $\pm$ .010 & .328 $\pm$ .013 & .269 $\pm$ .024 & .312 $\pm$ .043 \\
Mammo.   & .245 $\pm$ .022 & .231 $\pm$ .025 & .226 $\pm$ .010 & .176 & .229 & .212 & .254 & .190 $\pm$ .017 & .243 $\pm$ .005 & .188 $\pm$ .010 & .211 $\pm$ .020 \\
Seismic  & .159 $\pm$ .014 & .139 $\pm$ .013 & .100 $\pm$ .004 & .053 & .000 & .140 & .153 & .140 $\pm$ .014 & .126 $\pm$ .018 & .164 $\pm$ .003 & .129 $\pm$ .024 \\
Satimage & .371 $\pm$ .002 & .377 $\pm$ .004 & .377 $\pm$ .006 & .099 & .011 & .210 & .354 & .145 $\pm$ .002 & .344 $\pm$ .052 & .381 $\pm$ .000 & .387 $\pm$ .004 \\
Vowels   & .205 $\pm$ .025 & .198 $\pm$ .039 & .189 $\pm$ .017 & .407 & .228 & .585 & .146 & .284 $\pm$ .022 & .135 $\pm$ .020 & .298 $\pm$ .062 & .309 $\pm$ .082 \\
Musk     & .774 $\pm$ .001 & .773 $\pm$ .000 & .759 $\pm$ .003 & .024 & .000 & .040 & .773 & .531 $\pm$ .030 & .661 $\pm$ .081 & .278 $\pm$ .045 & .632 $\pm$ .122 \\
http     & .145 $\pm$ .000 & .145 $\pm$ .000 & .145 $\pm$ .000 & .000 & .145 & .006 & .014 & .145 $\pm$ .000 & .145 $\pm$ .000 & .145 $\pm$ .000 & .145 $\pm$ .000 \\
smtp     & .009 $\pm$ .000 & .009 $\pm$ .000 & .009 $\pm$ .000 & .009 & .009 & .009 & .009 & .009 $\pm$ .000 & .009 $\pm$ .000 & .009 $\pm$ .000 & .009 $\pm$ .012 \\ 
\bottomrule
\end{tabular}
}
\end{table*}

In Tables \ref{tab:mnist} and \ref{tab:fashion}, we present detailed results obtained for MNIST and Fashion-MNIST datasets.

\begin{table*}[tbh!]
\setlength{\tabcolsep}{2pt}
\caption{AUC obtained on MNIST dataset. The results of comparative methods (except DSVDD, LL-Flow and LL-Flow-Gen) are taken from \cite{wang2019multivariate}.} \label{tab:mnist}
\resizebox{\textwidth}{!}{%
\centering
\begin{tabular}{c | ccccccccccccccccc }
\toprule
Class & OCSVM  & VAE & DAE & LSA & GT & DAGM & GPND & DSEBM & DSVDD & NLL & TQM & LL-Flow & LL-Flow-Gen & OneFlow & OneFlow-Gen & OF-RNVP & OF-Glow\\
\midrule
0 & .995 & .985 & .982 & .998 & .982 & .500 & .999 & .320 & .980 & .995 & .993 & .990 & .990 & .994 & .995 & .600 & .934\\
1 & .999 & .997 & .998 & .999 & .893 & .766 & .999 & .987 & .997 & .998 & .997 & .998 & .998 & .999 & .999 & .997 & .997\\
2 & .926 & .943 & .936 & .923 & .993 & .326 & .980 & .482 & .917 & .953 & .948 & .930 & .927 & .945 & .944 & .523 & .913\\
3 & .936 & .916 & .929 & .974 & .987 & .319 & .968 & .753 & .919 & .963 & .957 & .952 & .936 & .967 & .972 & .683 & .896\\
4 & .967 & .945 & .940 & .955 & .993 & .368 & .980 & .696 & .949 & .966 & .963 & .937 & .941 & .954 & .949 & .712 & .909\\
5 & .955 & .929 & .928 & .966 & .994 & .490 & .987 & .727 & .885 & .962 & .960 & .961 & .949 & .973 & .974 & .652 & .828\\
6 & .987 & .977 & .982 & .992 & .999 & .515 & .998 & .954 & .983 & .992 & .990 & .984 & .985 & .990 & .991 & .591 & .911\\
7 & .966 & .975 & .971 & .969 & .966 & .500 & .988 & .911 & .946 & .969 & .966 & .970 & .971 & .976 & .975 & .824 & .948\\
8 & .903 & .864 & .857 & .935 & .974 & .467 & .929 & .536 & .939 & .955 & .951 & .841 & .851 & .870 & .874 & .545 & .729\\
9 & .962 & .967 & .974 & .969 & .993 & .813 & .993 & .905 & .965 & .977 & .976 & .965 & .970 & .972 & .972 & .701 & .883\\
\midrule
avg & .960 & .950 &  .950 &  .968 &  .977 &  .508 &  .982 &  .727 & .948 &  .973 &  .970 & .953 & .952 & .964 & .964 & .683 & .895\\
\bottomrule
\end{tabular}
}
\end{table*}

\begin{table*}[tbh!]
\setlength{\tabcolsep}{2pt}
\caption{AUC obtained on Fashion-MNIST dataset. The results of comparative methods (except DSVDD, LL-Flow and LL-Flow-Gen) are taken from \cite{wang2019multivariate}.} \label{tab:fashion}
\resizebox{\textwidth}{!}{%
\centering
\begin{tabular}{l | ccccccccccccccccc }
\toprule
Class & OCSVM & VAE & DAE & LSA & GT & DAGM & GPND & DSEBM & DSVDD & NLL & TQM & LL-Flow & LL-Flow-Gen & OneFlow & OneFlow-Gen & OF-RNVP & OF-Glow\\
\midrule
0 & .919 & .874 & .867 & .916 & .903 & .303 & .917 & .891 & .791 & .922 & .917 & .914 & .901 & .917 & .918 & .766 & .741\\
1 & .990 & .977 & .978 & .983 & .993 & .311 & .983 & .560 & .940 & .958 & .950 & .989 & .989 & .989 & .989 & .917 & .962\\
2 & .894 & .816 & .808 & .878 & .927 & .475 & .878 & .861 & .830 & .899 & .899 & .614 & .876 & .893 & .894 & .793 & .721\\
3 & .942 & .912 & .914 & .923 & .906 & .481 & .945 & .903 & .829 & .930 & .925 & .930 & .934 & .930 & .932 & .792 & .748\\
4 & .907 & .872 & .865 & .897 & .907 & .499 & .906 & .884 & .870 & .922 & .921 & .841 & .582 & .903 & .903 & .724 & .686\\
5 & .918 & .916 & .921 & .907 & .954 & .413 & .924 & .859 & .803 & .894 & .884 & .904 & .909 & .902 & .901 & .664 & .766\\
6 & .834 & .738 & .738 & .841 & .832 & .420 & .785 & .782 & .749 & .844 & .838 & .828 & .804 & .820 & .820 & .683 & .704\\
7 & .988 & .976 & .977 & .977 & .981 & .374 & .984 & .981 & .942 & .980 & .972 & .989 & .989 & .989 & .989 & .950 & .952\\
8 & .903 & .795 & .782 & .910 & .976 & .518 & .916 & .865 & .791 & .945 & .943 & .888 & .873 & .893 & .890 & .521 & .653\\
9 & .982 & .965 & .963 & .984 & .994 & .378 & .876 & .967 & .932 & .983 & .983 & .968 & .975 & .979 & .980 & .881 & .817\\
\midrule
avg & .928 & .884 & .881 & .922 & .937 & .472 & .911 & .855 & .847 & .928 & .923 & .887 & .883 & .922 & .922 & .769 & .775 \\
\bottomrule
\end{tabular}
}
\end{table*}

To illustrate previous results, in Figures \ref{fig:samples_mnist}, \ref{fig:samples_mnist2}, \ref{fig:samples_fashion}, \ref{fig:samples_fashion2} we show:
\begin{itemize}
\item examples from nominal class, which lie closest to the center of bounding hypersphere (1st column),
\item examples from nominal class, which are farthest from the center of bounding hypersphere (2nd column),
\item anomalies, which lie closest to the center of bounding hypersphere (3rd column),
\item anomalies, which are farthest from the center of bounding hypersphere (4th column),
\end{itemize}
Interestingly, the examples from the class "1" are frequently localized close to the center of bounding hypersphere. It partially explains the behavior observed in previous heatmaps. Another observation is that the examples closest to the center are very regular (first column), while examples farthest from the center look worse, and humans can make mistakes in classifying these images.

\begin{figure*}[tbh!] 
    \centering
        \includegraphics[width=0.95\textwidth]{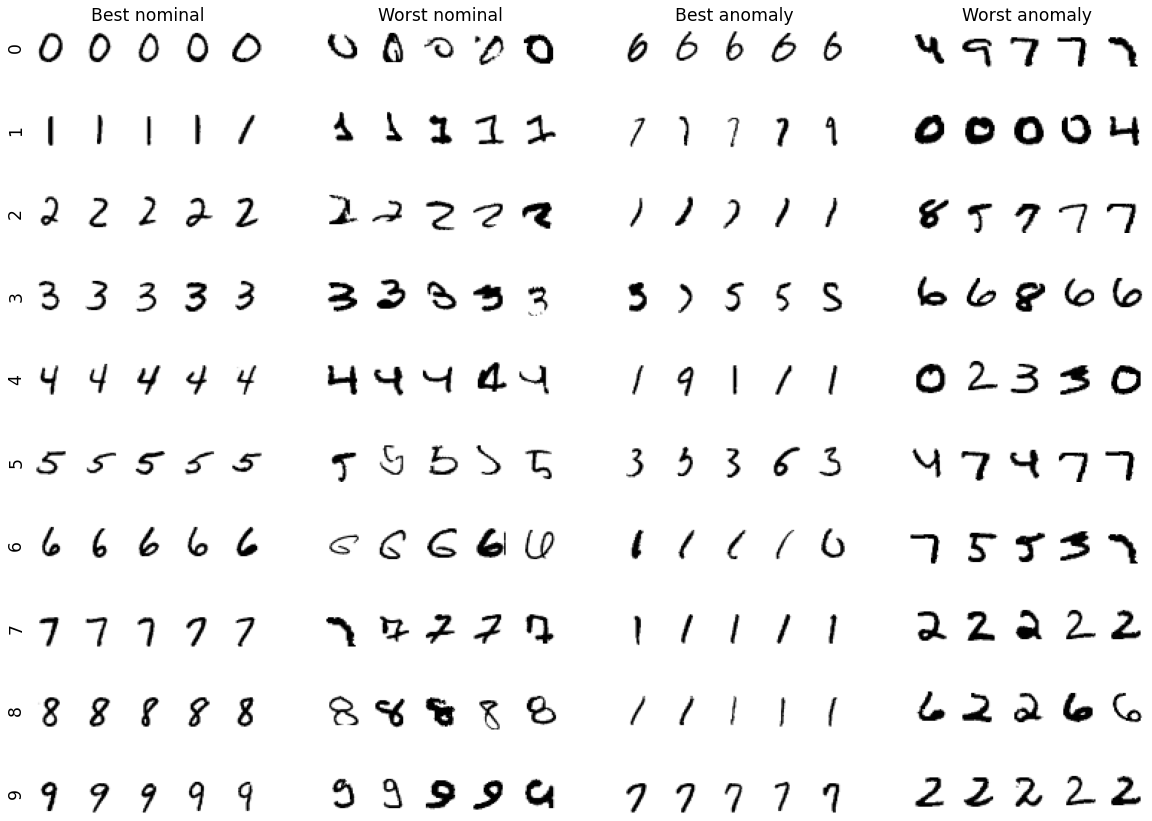}
    \caption{Samples for MNIST dataset for \our{} model.}
    \label{fig:samples_mnist}
\end{figure*}

\begin{figure*}[tbh!] 
    \centering
        \includegraphics[width=0.95\textwidth]{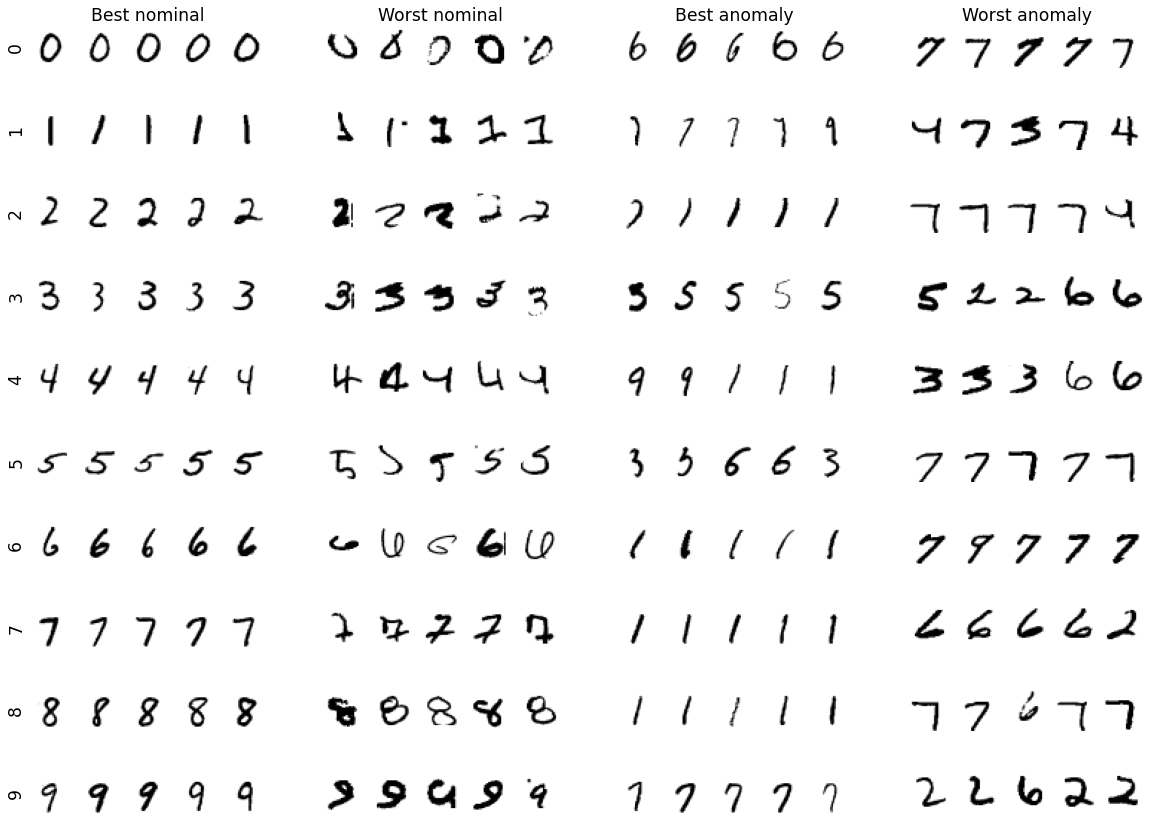}
    \caption{Samples for MNIST dataset for \ourdet{} model.}
    \label{fig:samples_mnist2}
\end{figure*}

\begin{figure*}[tbh!] 
    \centering
        \includegraphics[width=0.95\textwidth]{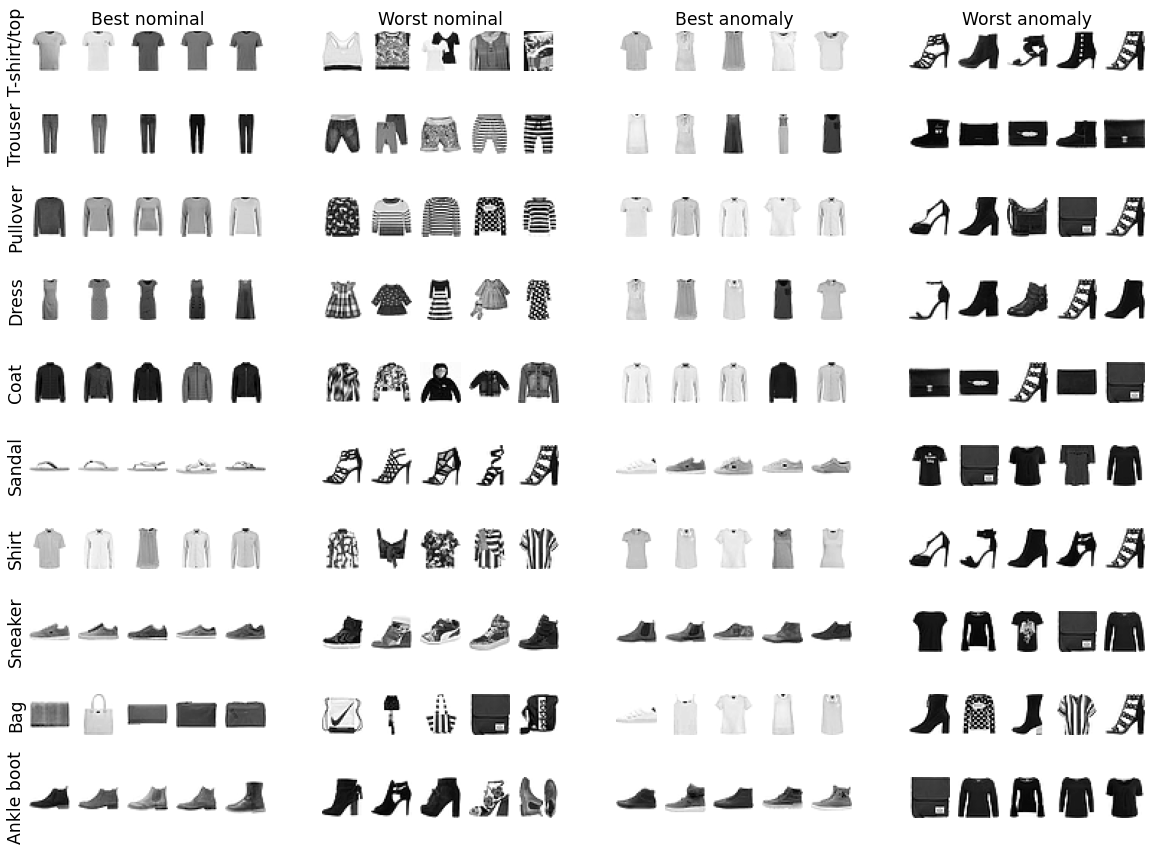}
    \caption{Samples for Fashion-MNIST dataset for \our{} model.}
    \label{fig:samples_fashion}
\end{figure*}

\begin{figure*}[tbh!] 
    \centering
        \includegraphics[width=0.95\textwidth]{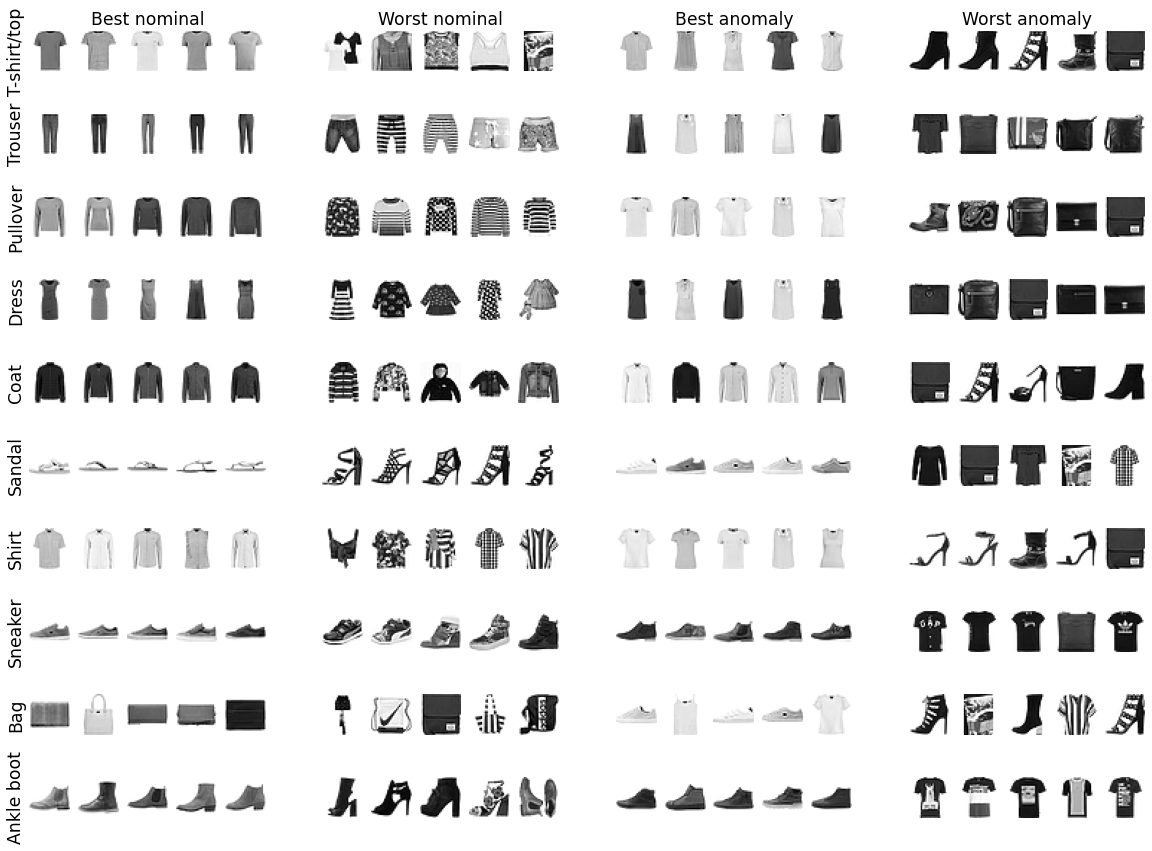}
    \caption{Samples for Fashion-MNIST dataset for \ourdet{} model.}
    \label{fig:samples_fashion2}
\end{figure*}

\end{document}